\documentclass[11pt]{article}

\usepackage[margin=2.5cm]{geometry}
\usepackage[utf8]{inputenc}
\usepackage{amsmath,amssymb,bm,amsthm}
\numberwithin{equation}{section}
\usepackage{booktabs,multirow,longtable}
\usepackage{float}
\usepackage[section]{placeins}
\usepackage{graphicx}
\usepackage[font=small,labelfont=bf]{caption}
\usepackage{subcaption}
\usepackage{xcolor}
\usepackage{enumitem}
\usepackage{microtype}
\usepackage[authoryear,round]{natbib}
\usepackage{tikz}
\usetikzlibrary{positioning,arrows.meta,fit,backgrounds,calc}
\usepackage[colorlinks=true,linkcolor=blue!60!black,citecolor=blue!60!black,urlcolor=blue!60!black]{hyperref}

\newcommand{\AutoScaleVersionPath}{}
\IfFileExists{revised/revised.tex}{\renewcommand{\AutoScaleVersionPath}{revised/}}{}
\graphicspath{{\AutoScaleVersionPath}}
\let\AutoScaleBibliography\bibliography
\renewcommand{\bibliography}[1]{\AutoScaleBibliography{\AutoScaleVersionPath#1}}

\newtheorem{theorem}{Theorem}[section]
\newtheorem{proposition}[theorem]{Proposition}
\newtheorem{lemma}[theorem]{Lemma}
\theoremstyle{remark}
\newtheorem{remark}[theorem]{Remark}
\theoremstyle{plain}

\newcommand{\method}{SWIM}
\newcommand{\scaleall}{SCALE-All}
\newcommand{\lewm}{LeWM}
\newcommand{\dinowm}{DINO-WM}
\newcommand{\vz}{\mathbf{z}}
\newcommand{\vq}{\mathbf{q}}

\newcommand{\vg}{\mathbf{g}}
\newcommand{\va}{\mathbf{a}}

\newcommand{\dq}{\boldsymbol{\delta}}
\newcommand{\Lc}[1]{\mathcal{L}_{\mathrm{#1}}}
\newcommand{\R}{\mathbb{R}}

\setlist{nosep,leftmargin=1.4em}
\providecommand{\showchanges}{0}
\newcommand{\rev}[1]{\ifnum\showchanges=1\textcolor{red}{#1}\else#1\fi}
\newenvironment{revision}{\begingroup\ifnum\showchanges=1\color{red}\fi}{\par\endgroup}
\title{\rev{\textbf{SWIM: Learning What Matters from Behaviors for JEPA World Models}}}

\author{\normalsize
Jiaming Hu\textsuperscript{1,2,*}, Shi Bo\textsuperscript{1,2,*},
Yan Zheng\textsuperscript{2}, Tian Wang\textsuperscript{2} \\[3pt]
\normalsize Florian Dubost\textsuperscript{2}, Alejandro Mottini\textsuperscript{2},
Junze Liu\textsuperscript{2}, Arvind Srinivasan\textsuperscript{2} \\[3pt]
\normalsize Kai Zhong\textsuperscript{2}, Kun Qian\textsuperscript{2},
Sharon Gao\textsuperscript{2}, Vianne Gao\textsuperscript{2}, Houyu Zhang\textsuperscript{2}, Qingjun Cui\textsuperscript{2} 
\\[6pt]
\normalsize \textsuperscript{1}Boston University \qquad
\textsuperscript{2}Unity Technologies
}
\date{\normalsize\today}
\newcommand{\blfootnote}[1]{%
  \begingroup\renewcommand\thefootnote{}\footnote{#1}%
  \addtocounter{footnote}{-1}\endgroup}

\begin{document}
\maketitle
\blfootnote{\textsuperscript{*}Equal contribution.
Correspondence to \texttt{jh7453@bu.edu}.
This work was completed while Jiaming Hu and Shi Bo were interns at Unity Technologies.}

\begin{abstract}
\begin{revision}
Joint-embedding predictive architectures (JEPAs) enable world models to plan from images
by predicting future embeddings. However, their visual prediction objectives do not explicitly
incorporate task-relevant state information, so the physical state that determines task progress
may not be reliably captured in the latent space. Supervising world models with simulator state can make planning
more informative, but it requires deciding which state features matter. Features are typically selected by hand or weighted uniformly, and a uniform weighting gives every standardized state feature equal influence, so
redundant coordinates can overcount the same variation, while irrelevant ones can obscure
task-relevant differences. To address this limitation, we propose SWIM (State Weighting with Inverse-Dynamics Model), which learns state weights from offline trajectories by predicting recorded actions from current states and observed future displacements. The learned weights quantify each feature's contribution to action prediction and define a weighted state metric for state-supervised training, which we demonstrate by aligning latent and state distances. The inverse-dynamics model is used only during training, so the
resulting world model plans from images with no additional planning-time computation. Across
five simulated robotic manipulation tasks, SWIM consistently improves planning success over
uniform state weighting.
\end{revision}
\end{abstract}

\newpage
\tableofcontents
\newpage

\section{Introduction}
\label{sec:introduction}
\begin{revision}
World models support planning by predicting the consequences of candidate actions before execution~\citep{ha2018world,hafner2019planet,chua2018pets}. Joint-embedding predictive architectures (JEPAs) perform this prediction in latent space without reconstructing future observations~\citep{lecun2022path,assran2023ijepa,zhou2025dinowm}. For example, LeWorldModel (LeWM) jointly learns an image encoder and action-conditioned latent dynamics through a next-embedding prediction objective and anti-collapse regularization~\citep{maes2026leworldmodel}. At planning time, candidate action sequences are evaluated using the Euclidean distance between predicted terminal and goal embeddings. However, the prediction objective does not explicitly require the latent space to represent the physical state that determines task progress. Although probing shows that JEPA latents encode physical quantities~\citep{maes2026leworldmodel}, forward prediction alone does not ensure that task-relevant physical state is reliably identifiable or well separated in latent space, which can limit downstream planning~\citep{yan2026psgjepa,zeng2026phylatent}.

This has motivated supervising world models with the simulator states recorded alongside offline data, which provide a direct signal of task progress: TC-WM aligns compact visual representations with proprioceptive state~\citep{fu2026parsimoniouslatentslearningtaskcentric}, PhyLatent and PSG-JEPA ground latent representations in physical state~\citep{zeng2026phylatent,yan2026psgjepa}, and SCALE aligns latent distances with task-state distances during world-model training~\citep{hu2026scale}. These methods take the supervising state as given, so their effectiveness depends on how the state is selected and weighted. Although state variables contain information relevant to planning, different coordinates may encode overlapping information, while others may be irrelevant to the task. Assigning equal weights to all coordinates can therefore overemphasize redundant information and introduce irrelevant variation into the state supervision, while selecting coordinates by hand requires task-specific expertise. This motivates learning the relative contributions of state features rather than treating all coordinates equally.

A key question is how to determine these weights. Prior action-based objectives, including inverse-dynamics models (which infer the action that produced an observed state change) and PLDM~\citep{agrawal2016poke,pathak2017curiosity,sobal2025pldm}, shape latent representations directly but do not address how individual state coordinates should be weighted. Since actions directly influence future state transitions and task outcomes, we leverage recorded actions to identify state changes that are informative about behavior: a coordinate whose change helps infer the action reflects the agent's behavior, whereas a redundant or unaffected coordinate adds little information. This provides a behavior-grounded surrogate for state relevance.

Building on this idea, we propose \method{} (State Weighting with Inverse-Dynamics Model), a two-stage framework for learning behavior-informed state supervision. First, a gated inverse-dynamics model learns continuous state weights by predicting recorded actions from current states and observed future displacements. Its training combines the log-likelihood of recorded actions, discrimination between matching and mismatched state changes, and regularization. The learned weights define a weighted state metric that can replace the hand-specified state in state-supervised training; we then demonstrate it by calibrating JEPA latent geometry through SCALE~\citep{hu2026scale}. This approach requires neither reward nor success labels. Because the inverse-dynamics model and simulator states are used only during training, the world model still plans from images without additional inference-time computation.

Our contributions are as follows: \textbf{(i)} We introduce an inverse-dynamics approach to automatically learn the relative importance of state features from offline trajectories. \textbf{(ii)} We demonstrate the learned state metric in SCALE, aligning JEPA latent distances with weighted state distances to better support downstream planning. \textbf{(iii)} We evaluate \method{} on five OGBench manipulation tasks~\citep{park2025ogbench} with the cross-entropy method (CEM) and iCEM planners over five training seeds, where it consistently improves over uniform state weighting , and we analyze how the learned weights respond to structured distractors.

The rest of the paper is organized as follows. Section~\ref{sec:related} reviews related work, and Section~\ref{sec:preliminaries} introduces latent-space planning and state-distance supervision. Section~\ref{sec:method} presents \method{}, followed by the experiments in Section~\ref{sec:experiments}. Section~\ref{sec:discussion} discusses limitations and concludes. The appendices provide implementation and evaluation details, an analysis of the state transformation and learned weights, and additional experimental results.
\end{revision}

\section{Related Work}
\label{sec:related}
\begin{revision}
\paragraph{Latent World Models and Planning.}
World models support control by learning policies from imagined experience~\citep{ha2018world,hafner2020dreamer} or optimizing action sequences through model predictions~\citep{chua2018pets,hafner2019planet,hansen2024tdmpc2}. Joint-embedding predictive architectures (JEPAs) learn representations through latent prediction without reconstructing observations~\citep{lecun2022path,assran2023ijepa}. DINO-WM builds on pretrained visual features~\citep{zhou2025dinowm}, whereas PLDM and LeWorldModel (LeWM) jointly learn visual representations and latent dynamics~\citep{sobal2025pldm,maes2026leworldmodel}. Their prediction objectives do not specify which task-relevant information the representation should capture, which motivates the state supervision discussed next.

\paragraph{State Supervision and Planning-Oriented Representations.}
Several approaches introduce additional supervision to improve the planning relevance of learned representations. TC-WM learns a compact latent space with a subspace contrastively aligned to the agent's physical state~\citep{fu2026parsimoniouslatentslearningtaskcentric}, while PhyLatent and PSG-JEPA incorporate physical-state grounding and dynamics-related objectives~\citep{zeng2026phylatent,yan2026psgjepa}. Other methods use value guidance, temporal distance, and reachability information to make latent representations more suitable for planning~\citep{destrade2025valuejepa,bai2026tdjepa,li2026rcaux}. Decision-Metric Alignment further studies whether latent distances correctly rank candidate action sequences~\citep{wang2026decisionmetric}. SCALE directly correlates pairwise latent distances with task-state distances~\citep{hu2026scale}. TC-WM, PhyLatent, PSG-JEPA, and SCALE all rely on a fixed choice of state coordinates, selected by hand or weighted uniformly; \method{} instead learns how much each state coordinate should contribute, and we demonstrate it with SCALE.

\paragraph{Inverse Dynamics and Behavioral Relevance.}
Inverse dynamics has been widely used to learn action-relevant representations from observed transitions~\citep{agrawal2016poke,pathak2017curiosity}. More recent approaches incorporate action-based objectives into latent world-model training, including PLDM, DA-LeWM, AC-MTM, and LeWAM~\citep{sobal2025pldm,wang2026decisionmetric,boylan2026nogaussian,hegde2026lewam}. Related work on actionable representations and bisimulation connects state similarity to behavior and task-related signals~\citep{ghosh2019actionable,ferns2004bisim,zhang2021dbc}, while sparse regularization and differentiable feature selection provide mechanisms for learning feature weights~\citep{tibshirani1996lasso,louizos2018l0}. These methods shape latent representations directly, define state similarity through rewards or learned policies, or provide weighting mechanisms without a behavioral target; none learns per-coordinate weights over recorded state variables for state supervision. \method{} instead uses gated inverse dynamics to learn weights over recorded state variables, which then define the state supervision for world-model training.
\end{revision}

\section{Preliminaries}
\label{sec:preliminaries}
\begin{revision}
In this section, we review latent-space planning with LeWM and introduce SCALE, the state-supervised objective we use to demonstrate SWIM, which correlates latent distances with state distances. Throughout the paper, lowercase bold letters denote vectors and uppercase bold letters denote matrices; unbolded indexed letters denote scalar entries. For example, $\mathbf v=(v_1,\ldots,v_d)^\top\in\R^d$ has entries $v_k$, and $\mathbf M=(M_{ij})\in\R^{m\times n}$ has entries $M_{ij}$. For $\mathbf v\in\R^d$, we define $\lVert\mathbf v\rVert_1=\sum_{k=1}^{d}|v_k|$ and $\lVert\mathbf v\rVert_2=(\sum_{k=1}^{d}v_k^2)^{1/2}$.

\subsection{Latent-Space Planning with LeWM}
LeWorldModel (LeWM) jointly learns an image encoder and action-conditioned latent dynamics, enabling an agent to evaluate candidate action sequences by predicting their outcomes in latent space~\citep{maes2026leworldmodel}. We introduce its planning formulation and training objective below. Let $o_t$ denote an image at decision step $t$ and $\va_t$ the stacked vector of actions executed before the next decision step. An episode is one recorded environment trajectory, and a sub-trajectory is a contiguous segment of an episode. LeWM uses an encoder $E_\theta$ and an action-conditioned predictor $P_\psi$, with trainable parameters $\theta$ and $\psi$, respectively. Writing $d_z$ for the latent dimension, the model produces $\vz_t=E_\theta(o_t)\in\R^{d_z}$ and $\hat\vz_{t+1}=P_\psi(\vz_t,\va_t)$. We suppress the predictor's finite history in this notation.

At test time, $H$ is the planning horizon in decision steps. Given a current image $o_0$ and candidate actions $\va_{0:H-1}$, the model initializes $\hat\vz_0=E_\theta(o_0)$ and rolls out
\begin{equation}
  \hat\vz_{t+1}=P_\psi(\hat\vz_t,\va_t),\qquad t=0,\ldots,H-1.
  \label{eq:rollout}
\end{equation}
For a fixed initial observation and history, $\hat\vz_H(\va_{0:H-1})$ denotes the predicted terminal embedding under that action sequence. Given a goal image $o_g$ and its embedding $\vz_g=E_\theta(o_g)$, we define the planning cost function $J$, which assigns a scalar score to each candidate action sequence, as
\begin{equation}
  J(\va_{0:H-1})=\bigl\lVert\hat\vz_H(\va_{0:H-1})-\vz_g\bigr\rVert_2^2 .
  \label{eq:cost}
\end{equation}
The cross-entropy method (CEM) and its improved variant iCEM search for action sequences that minimize the planning cost $J$ in Equation~\eqref{eq:cost}, while the encoder and predictor remain fixed~\citep{deboer2005cem,pinneri2021icem}. Thus the solver performs the search, and the latent distance supplies the cost used to compare candidates.

LeWM learns the encoder and predictor from observed transitions. For a transition $(o_t,\va_t,o_{t+1})$, the one-step latent-prediction loss is
\begin{equation}
  \Lc{pred}=\lVert\hat\vz_{t+1}-\vz_{t+1}\rVert_2^2,
  \qquad \vz_{t+1}=E_\theta(o_{t+1}).
  \label{eq:prediction}
\end{equation}
This loss alone admits collapsed solutions in which every image has the same embedding. LeWM stacks a batch of encoder outputs as the rows of a matrix $\mathbf{Z}$ and applies SIGReg, which projects the embeddings onto random unit directions and penalizes deviations of those projected distributions from a standard Gaussian~\citep{balestriero2025lejepa,maes2026leworldmodel}. This regularizer maintains global spread. With nonnegative weight $\lambda_{\mathrm{sig}}$, the objective is
\begin{equation}
  \Lc{LeWM}=\Lc{pred}+\lambda_{\mathrm{sig}}\,\mathrm{SIGReg}(\mathbf{Z}).
  \label{eq:lewm}
\end{equation}
While these objectives support stable latent prediction, neither explicitly optimizes the representation geometry used to compare candidate action sequences in Equation~\eqref{eq:cost}.

\subsection{SCALE}
SCALE (State-CAlibrated Latent Embeddings) aims to make the planning cost more informative about task progress by encouraging latent distances to reflect differences in task-relevant state~\citep{hu2026scale}. Let $\vq\in\R^d$ contain selected state variables, such as an object's position and orientation, standardized componentwise using statistics computed once over all frames of the offline dataset, before the train/validation split. For $K$ sampled image-frame index pairs $(i_\ell,j_\ell)$, $\ell=1,\ldots,K$, latent and state distances are computed from the same pairs:
\begin{equation}
  x_\ell=\lVert\vz_{i_\ell}-\vz_{j_\ell}\rVert_2^2,\qquad
  y_\ell=\lVert\vq_{i_\ell}-\vq_{j_\ell}\rVert_2^2 .
  \label{eq:profiles}
\end{equation}
Writing $\mathbf x=(x_1,\ldots,x_K)^\top$ and $\mathbf y=(y_1,\ldots,y_K)^\top$ for the paired distance vectors, SCALE encourages positive correlation between them. For nonconstant $\mathbf x$ and $\mathbf y$, its correlation loss and complete training objective take the following form when numerical stabilization is omitted:
\begin{equation}
  \Lc{corr}=1-\rho_{\mathrm P}(\mathbf x,\mathbf y),\qquad
  \Lc{SCALE}=\Lc{pred}+\lambda_{\mathrm{sig}}\,\mathrm{SIGReg}(\mathbf{Z})
                   +\lambda_{\mathrm{corr}}\,\Lc{corr},
  \label{eq:scale}
\end{equation}
where $\rho_{\mathrm P}$ denotes Pearson correlation and the nonnegative coefficients weight the regularizers. Minimizing $\Lc{corr}$ encourages pairs with larger task-state differences to lie farther apart in latent space. 

We refer to SCALE applied to all supplied state coordinates with uniform weights as \scaleall{}; it is the main baseline in our experiments. SCALE uses state distances as fixed supervision targets for latent geometry. However, these targets are determined by the supplied state coordinates and their uniform weighting. Consequently, redundant or irrelevant coordinates can influence the learned geometry even when latent-state alignment is successful. This limitation motivates learning a weighted state metric before using it for representation calibration. In the next section, we introduce SWIM, which learns these weights from offline action information.

\end{revision}
\section{SWIM}
\label{sec:method}
\label{sec:swim}
\label{sec:meaning}

\begin{revision}
SWIM consists of two stages (Figure~\ref{fig:pipeline}). In Stage~1, we learn
state-feature weights from offline trajectories using gated inverse dynamics;
these weights define a state metric that does not depend on the downstream objective.
In Stage~2, we demonstrate this metric with SCALE, using it as the target for calibrating
the latent geometry of JEPA world models.
\end{revision}

\begin{figure}[!t]
  \centering
  \includegraphics[width=\textwidth]{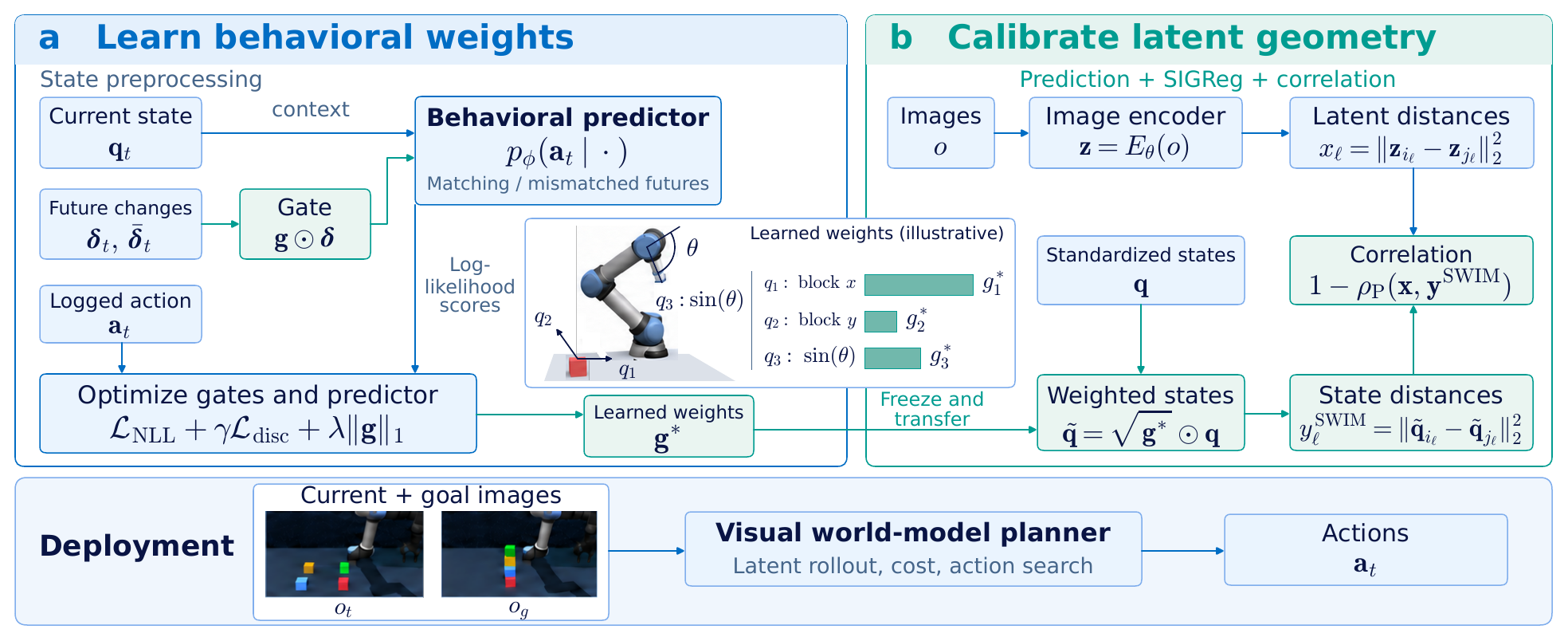}
  \caption{\rev{\textbf{\method.} (a)~The behavioral predictor learns coordinate weights from current states, gated future changes, and logged actions. (b)~Frozen weights rescale standardized states. Squared latent and state distances from the same frame pairs enter the correlation loss, with state distances held fixed; prediction and SIGReg losses are retained. The inset names physical features before standardization and shows illustrative weight bars; within this inset, $\theta$ denotes a joint angle. The training symbols are defined in Sections~\ref{sec:stage1}--\ref{sec:stage2}. Deployment uses image-based latent planning; the behavioral predictor and logged future states are training-only components.}}
  \label{fig:pipeline}
\end{figure}

\subsection{Learning State-Feature Gates via Inverse Dynamics}
\label{sec:swim_gates}
\label{sec:stage1}
\label{sec:swim_gate_learning}

\begin{revision}
Each Stage-1 example is drawn from one offline trajectory and
contains the state $\mathbf q_t\in\mathbb R^d$ at a decision
step $t$, the state $\mathbf q_{t+h}$ reached $h$ decision
steps later in the same trajectory, and the action chunk
$\mathbf a_t$ recorded at step $t$. The behavioral horizon
$h$ is sampled per example and is distinct from the planning horizon
$H$ (Appendix~\ref{app:stage1}).
As in SCALE, each state coordinate is standardized by its mean and
standard deviation over all frames of the offline dataset; Stage~1
additionally floors the standard deviation and clips extreme values
(Appendix~\ref{app:stage1}). Standardization places coordinates with
different physical units, such as positions and joint-angle features,
on a common scale, so each gate weights a coordinate relative to its
typical variation in the data rather than its units.

We introduce trainable gate logits
$\boldsymbol{\alpha}\in\mathbb R^d$ and define
$\mathbf g=\operatorname{sigmoid}(\boldsymbol{\alpha})
\in(0,1)^d$.
We define the observed displacement
$\boldsymbol\delta_t=\mathbf q_{t+h}-\mathbf q_t$.
A Gaussian multilayer perceptron (MLP) parameterized by $\phi$, which we call the behavioral predictor, predicts
the recorded action from the current state and a gated displacement
$\boldsymbol\delta$, and we write the resulting log-likelihood as
\begin{equation}
s_t(\boldsymbol\delta)
=\log p_\phi\bigl(\mathbf a_t\mid[\mathbf q_t;\,
\mathbf g\odot\boldsymbol\delta]\bigr),
\label{eq:swim_input}
\end{equation}
where $[\cdot\,;\cdot]$ denotes concatenation. The MLP outputs the
mean and per-component log standard deviations of the action chunk
(Appendix~\ref{app:stage1}).
The current state remains ungated, while each gate
scales the corresponding displacement coordinate.
A single gate vector is shared across trajectories
within an environment.

Fitting the recorded actions gives the negative log-likelihood (NLL)
\begin{equation}
\mathcal L_{\mathrm{NLL}}
=-\mathbb E\bigl[s_t(\boldsymbol\delta_t)\bigr],
\label{eq:swim_nll}
\end{equation}
where the expectation is over offline examples
$(\mathbf q_t,\mathbf q_{t+h},\mathbf a_t)$.
This objective alone does not guarantee that the predictor uses the
displacement, because actions can often be partly predicted from the
current state. If the predictor ignores the displacement, the NLL
provides no gradient to the gates, and the gate penalty introduced
below drives all gates toward zero regardless of behavioral relevance.
We therefore add a discrimination term that requires the observed
displacement $\boldsymbol\delta_t$ to explain the recorded action
better than mismatched displacements
$\bar{\boldsymbol\delta}_t=\bar{\mathbf q}-\mathbf q_t$, where
$\bar{\mathbf q}$ is the later state of another example in the same
training batch, and $\mathbf q_t$ and $\mathbf a_t$ are held fixed.
We consider two forms. The hinge loss uses a single mismatched
displacement:
\begin{equation}
\mathcal L_{\mathrm{hinge}}
=\mathbb E\bigl[\max\bigl(0,\,\kappa-s_t(\boldsymbol\delta_t)
+s_t(\bar{\boldsymbol\delta}_t)\bigr)\bigr],
\label{eq:swim_losses}
\end{equation}
where $\kappa>0$ is the margin. If the predictor ignores the displacement,
then $s_t(\boldsymbol\delta_t)=s_t(\bar{\boldsymbol\delta}_t)$ and the
loss remains $\kappa$. The InfoNCE loss~\citep{oord2018cpc} instead
contrasts the observed displacement with $N$ mismatched displacements
$\bar{\boldsymbol\delta}_{t,k}=\bar{\mathbf q}_k-\mathbf q_t$,
$k=1,\ldots,N$:
\begin{equation}
\mathcal L_{\mathrm{InfoNCE}}
=-\mathbb E\Biggl[
\log\frac{\exp\bigl(s_t(\boldsymbol\delta_t)/\tau\bigr)}
{\exp\bigl(s_t(\boldsymbol\delta_t)/\tau\bigr)
+\sum_{k=1}^{N}\exp\bigl(s_t(\bar{\boldsymbol\delta}_{t,k})/\tau\bigr)}
\Biggr],
\label{eq:infonce}
\end{equation}
where $\tau>0$ is a temperature. Appendix~\ref{app:stage1} specifies
how the mismatched later states are drawn from the batch.
The complete Stage-1 objective uses a discrimination weight
$\gamma\ge0$ and a gate-penalty coefficient $\lambda\ge0$:
\begin{equation}
\mathcal L_{\mathrm{Stage1}}(\boldsymbol\alpha,\phi)
=\mathcal L_{\mathrm{NLL}}
+\gamma\mathcal L_{\mathrm{disc}}
+\lambda\|\mathbf g\|_1,
\qquad
\mathcal L_{\mathrm{disc}}\in\{\mathcal L_{\mathrm{hinge}},\mathcal L_{\mathrm{InfoNCE}}\}.
\label{eq:swim_stage1}
\end{equation}
We jointly optimize $\boldsymbol\alpha$ and $\phi$; the $\ell_1$
penalty acts on $\mathbf g$ and encourages small gates without
hard feature selection. Sampling and preprocessing
details are in Appendix~\ref{app:stage1}.
\end{revision}

\subsection{Transferring Learned Gates to Latent Geometry}
\label{sec:swim_transfer}
\label{sec:problem}
\label{sec:stage2}
\label{sec:swim_calibration}

\begin{revision}
We freeze the learned gates $\mathbf g^*$ and apply them to the same
state coordinates after SCALE's standardization. The gates must enter
after standardization and must not be followed by a second
coordinate-wise standardization, which would undo them
(Appendix~\ref{app:metric}). Defining
$\widetilde{\mathbf q}=\sqrt{\mathbf g^*}\odot\mathbf q$
(with elementwise square root), we obtain the weighted squared
state distance
\begin{equation}
D_{\mathbf g^*}^2(\mathbf q_i,\mathbf q_j)
=\|\widetilde{\mathbf q}_i-\widetilde{\mathbf q}_j\|_2^2
=\sum_{k=1}^{d}g_k^*(q_{i,k}-q_{j,k})^2.
\label{eq:swim_distance}
\end{equation}

For the $K$ frame pairs $(i_\ell,j_\ell)$ used by SCALE
(Equation~\eqref{eq:profiles}), we replace SCALE's uniform state-distance target with
$y_\ell^{\mathrm{SWIM}}=
D_{\mathbf g^*}^2(\mathbf q_{i_\ell},\mathbf q_{j_\ell})$.
Let $\mathbf y^{\mathrm{SWIM}}$ collect these targets and
$\mathbf x$ the corresponding squared latent distances.
Substituting the new target into SCALE's correlation term gives
\begin{equation}
\mathcal L_{\mathrm{SWIM}}
=\mathcal L_{\mathrm{LeWM}}
+\lambda_{\mathrm{corr}}
\bigl[1-\rho_{\mathrm P}(\mathbf x,\mathbf y^{\mathrm{SWIM}})\bigr].
\label{eq:swim_stage2}
\end{equation}
Here $\mathcal L_{\mathrm{LeWM}}$ is the objective from Equation~\eqref{eq:lewm}.
Relative to \scaleall{}, only the state-distance targets change;
$\mathbf g^*$ remains fixed and the correlation loss trains the
image encoder. The inverse-dynamics model is discarded after Stage~1.
At deployment, planning uses the original image-based latent cost
and optimizer, without simulator states or additional computation.
\end{revision}

\subsection{Why Learned Weights Help Planning}
\label{sec:theory}

\begin{revision}
The state metric determines the geometry toward which the latent space is calibrated. Under uniform weighting, a redundant encoding of one quantity, such as a sine--cosine pair, is counted more than once, and coordinates unrelated to the task add variation to the target; even perfect latent--state alignment then carries these distortions into the planning cost. The following result makes this precise by separating the planning error into an alignment term, a rollout term, and a term that depends on the state weights.
\end{revision}

\begingroup
\color{black}
Following SCALE's planning analysis~\citep{hu2026scale}, fix a start--goal
pair and $M\ge2$ candidate action sequences. For candidate $m$, let
$\vq_H^{(m)}\in\R^d$ be its actual terminal state, $\vz_H^{(m)}$ the encoded
actual terminal observation, and $\widehat\vz_H^{(m)}$ its predicted terminal
embedding; $\vq_g$ and $\vz_g$ denote the goal state and goal embedding. Define
\[
J_m^{\mathrm{true}}=\lVert\vz_H^{(m)}-\vz_g\rVert_2^2,
\qquad
\widehat J_m=\lVert\widehat\vz_H^{(m)}-\vz_g\rVert_2^2,
\qquad
Q_m=D_{\vg}^2(\vq_H^{(m)},\vq_g),
\]
where $\widehat J_m$ is the planning cost in Equation~\eqref{eq:cost},
$\vg=\vg^*/\lVert\vg^*\rVert_1$ are the normalized learned gates, and the
subscript of $D^2$ indicates the weights used in
Equation~\eqref{eq:swim_distance}. Task progress is measured by a reference
metric with weights $\vg^0\ge0$, $\sum_k g_k^0=1$, on the same standardized
state coordinates. Let $\sigma_z$ and $\sigma_q$ be the empirical standard deviations of
$(J_m^{\mathrm{true}})_{m=1}^M$ and $(Q_m)_{m=1}^M$, respectively, using
divisor $M$, and let $\rho_{\mathrm P}$ be their Pearson correlation.

\begin{theorem}[Candidate selection with learned state weights]
\label{thm:swim-weighted-planning}
Suppose $\sigma_z,\sigma_q>0$, and let $C_z=\max_m\lVert\vz_H^{(m)}-\vz_g\rVert_2,\
C_{\mathrm{roll}}=\max_m\lVert\widehat\vz_H^{(m)}-\vz_H^{(m)}\rVert_2,
R=\max_{m,k}|q_{H,k}^{(m)}-q_{g,k}|.$ For every $\widehat m\in\arg\min_m\widehat J_m$, we have
\begin{align}
0\le{}&D_{\vg^0}^2(\vq_H^{(\widehat m)},\vq_g)
-\min_m D_{\vg^0}^2(\vq_H^{(m)},\vq_g)\notag\\
\le{}&2\sigma_q\sqrt{M(1-\rho_{\mathrm P})}
+2\frac{\sigma_q}{\sigma_z}\bigl(2C_zC_{\mathrm{roll}}+C_{\mathrm{roll}}^2\bigr)
+R^2\lVert\vg-\vg^0\rVert_1.
\label{eq:swim-weighted-planning}
\end{align}
\end{theorem}

The three terms in Equation~\eqref{eq:swim-weighted-planning} correspond to
imperfect alignment with the learned metric, latent rollout error, and
discrepancy from the reference task weights. The last term vanishes when
$\vg=\vg^0$. \scaleall{} corresponds to uniform weights
$\vg^{\mathrm u}=(1/d,\ldots,1/d)$, for which the bound retains the term
$R^2\lVert\vg^{\mathrm u}-\vg^0\rVert_1$ however well the latent space is aligned.
Learned weights that move toward the reference metric reduce this term;
Proposition~\ref{prop:swim-weight-advantage} and
Remark~\ref{rem:swim-scale-bound} in
Appendix~\ref{app:swim-weighted-planning-proof} make this comparison precise.

\par
\endgroup

\section{Experiments}
\label{sec:experiments}
\subsection{Setup}
\label{sec:setup}
\begin{revision}
\paragraph{Environments and Comparisons.}
We evaluate five OGBench visual manipulation tasks~\citep{park2025ogbench}: Cube, Cube-Double, Cube-Triple, Cube-Quadruple, and Scene. Their simulator states expose structured task variables, including object poses, end-effector and joint configurations, and articulated scene elements, together with redundant encodings such as sine--cosine pairs, and the Cube tasks increase the state dimension from 22 to 37 as objects are added. This makes them well suited for studying state weighting. All candidate state coordinates are supplied to the gate together, so feature construction is fixed while relevance weighting is learned. We compare LeWM, LeWM with \scaleall, and two \method{} variants, SWIM-H and SWIM-N, trained with hinge and InfoNCE discrimination, respectively. We also evaluate \dinowm{}, which plans with pretrained visual features. All methods use the same offline data, planning protocol, and planning instances.

\paragraph{Evaluation.}
We evaluate with CEM and iCEM. Each \lewm{}-based method is trained with five seeds, and each task, configuration, and planner is evaluated on 600 episodes per training seed. Unless stated otherwise, SWIM-H uses $\lambda=0.03$ and SWIM-N uses $\lambda=0.1$. Table~\ref{tab:main} reports CEM results and Table~\ref{tab:icem} reports iCEM results. For significance, we average success over all tasks and both planners for each training seed and compare with \scaleall{} using a paired $t$-test over the five seeds. Planner parameters, episode allocation, and goal sampling are detailed in Appendix~\ref{app:planning}. The appendices provide gate-learning and feature-weight analyses (Appendices~\ref{app:stage1}--\ref{app:gate_values}), structured-distractor diagnostics (Appendix~\ref{app:contamination}), and objective and data-efficiency ablations (Appendix~\ref{app:ablations}).
\end{revision}

\subsection{Does Behavioral Weighting Improve Planning?}
\label{sec:main}
\begin{table}[!htbp]
  \normalcolor
  \centering
  \footnotesize
  \setlength{\tabcolsep}{3.5pt}
  \begin{tabular}{lrrrrrrr}
    \toprule
    \textbf{Task} & \textbf{\dinowm$^\dagger$} & \textbf{\lewm} & \textbf{\scaleall} & \textbf{SWIM-H$^\star$} & \textbf{SWIM-N$^\star$} & $\Delta_{\text{\scaleall}}$ & $\Delta_{\text{\lewm}}$ \\
    \midrule
    Cube           & 72.5 & 67.7{\scriptsize$\pm$1.0} & 74.6{\scriptsize$\pm$1.1} & \textbf{75.9}{\scriptsize$\pm$1.5} {\scriptsize(.01)} & \underline{75.8}{\scriptsize$\pm$1.2} {\scriptsize(.03)} & +1.3 & +8.3 \\
    Cube-Double    & 17.5 & 16.4{\scriptsize$\pm$1.1} & 18.2{\scriptsize$\pm$0.8} & \textbf{19.4}{\scriptsize$\pm$1.1} {\scriptsize(1.0)} & \underline{19.3}{\scriptsize$\pm$0.7} {\scriptsize(1.0)} & +1.1 & +3.0 \\
    Cube-Triple    & \textbf{34.2} & 30.6{\scriptsize$\pm$0.8} & 31.3{\scriptsize$\pm$1.2} & \underline{32.6}{\scriptsize$\pm$1.0} {\scriptsize(.05)} & 32.4{\scriptsize$\pm$0.5} {\scriptsize(.03)} & +1.3 & +2.0 \\
    Cube-Quadruple & 30.2 & 28.3{\scriptsize$\pm$1.0} & 29.6{\scriptsize$\pm$0.7} & \textbf{30.7}{\scriptsize$\pm$0.9} {\scriptsize(.03)} & \underline{30.5}{\scriptsize$\pm$0.6} {\scriptsize(1.0)} & +1.0 & +2.3 \\
    Scene          & 22.0 & 30.5{\scriptsize$\pm$0.5} & 33.4{\scriptsize$\pm$1.9} & \underline{37.3}{\scriptsize$\pm$3.3} {\scriptsize(.1)} & \textbf{39.4}{\scriptsize$\pm$2.3} {\scriptsize(.1)} & +6.0 & +8.9 \\
    \bottomrule
  \end{tabular}
  \caption{\rev{\textbf{CEM planning success (\%) over five training seeds} (mean $\pm$ standard deviation). SWIM-H$^\star$ and SWIM-N$^\star$ use the best $\lambda$ for each task, shown in parentheses. \dinowm$^\dagger$ uses pretrained visual features. The $\Delta$ columns compare the better SWIM variant with \scaleall{} and \lewm. Bold: best; underline: second best.}}
  \label{tab:main}
\end{table}
\begin{revision}
\method{} outperforms \lewm{} and \scaleall{} on all five tasks (Table~\ref{tab:main}), improving over \scaleall{} by 1.0--6.0 percentage points and over \lewm{} by 2.0--8.9 points, with the largest gain on Scene. With the default configurations, SWIM-H and SWIM-N improve over \scaleall{} by 0.9 points ($p=0.009$) and 1.0 points ($p=0.022$), respectively, and averaging over all ten configurations gives 0.7 points ($p=0.006$); in each case, \method{} is better on all five training seeds. \dinowm{} remains stronger on Cube-Triple. SWIM-H$^\star$ is higher on four tasks and SWIM-N$^\star$ on Scene, but the two objectives perform comparably overall, so the choice of discrimination loss is not critical.
\end{revision}

\subsection{What Happens When Irrelevant Variables Are Supplied?}
\label{sec:contam}
\begin{figure}[t]
  \centering
  \includegraphics[width=\textwidth]{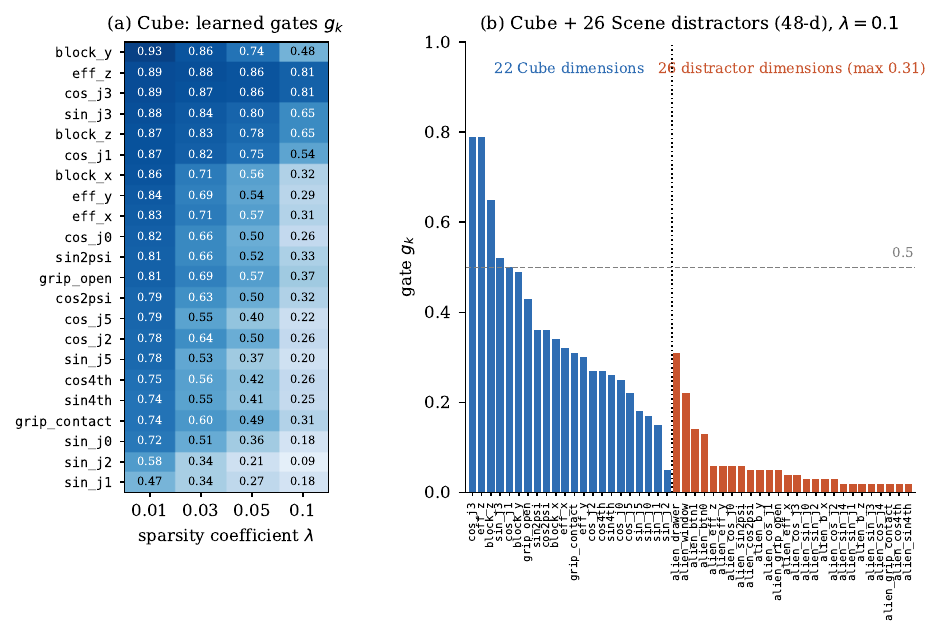}
  \caption{\textbf{Learned gates on Cube (InfoNCE with $\lVert\vg\rVert_1$ regularization).} (a)~Gate values for the 22 candidate state dimensions as the sparsity coefficient $\lambda$ increases. End-effector height (\texttt{eff\_z}), block height (\texttt{block\_z}), and the arm joints (\texttt{cos\_j1}, \texttt{cos\_j3}, \texttt{sin\_j3}) are retained longest, while \texttt{sin\_j1} and \texttt{sin\_j2} are suppressed first. (b)~Gates learned at $\lambda=0.1$ when 26 structured variables from unrelated Scene trajectories are appended to the Cube state (48 candidate dimensions). Every distractor receives a gate of at most 0.31, and the five Cube dimensions at or above 0.5 are exactly the five retained at $\lambda=0.1$ on the clean 22-dimensional state.}
  \label{fig:gates}
\end{figure}
\begin{revision}
\method{} assigns low weights to irrelevant state variables while preserving the ranking of the original coordinates, and largely recovers the planning success lost to them. We first examine the weights on the clean Cube state. Figure~\ref{fig:gates}(a) shows how regularization changes the Cube weights. At $\lambda=0.1$, five coordinates have gates of at least 0.5: \texttt{eff\_z}, \texttt{cos\_j1}, \texttt{cos\_j3}, \texttt{sin\_j3}, and \texttt{block\_z}. Relative rankings can change with regularization: for example, \texttt{block\_y} starts highest but falls below 0.5 at the strongest displayed penalty.

To test contamination, we append 26 variables from unrelated Scene trajectories to the 22-dimensional Cube vector. The ranking of the original coordinates remains similar to the clean-state ranking across the reported penalties (Table~\ref{tab:stability}); Figure~\ref{fig:gates}(b) displays all weights in the augmented state at $\lambda=0.1$. At $\lambda=0.1$, their Spearman correlation is 0.995, and all appended coordinates have gates at most 0.31.

The distractors also affect planning (Table~\ref{tab:contam_plan}, single training seed). Uniform alignment to the contaminated state lowers CEM success from 75.5 to 74.2, and weighted alignment at $\lambda=0.05$ recovers 75.2. Under iCEM, weighted alignment reaches 76.7, above both contaminated (75.5) and clean (76.3) uniform alignment.
\end{revision}
\section{Discussion and Conclusion}
\label{sec:discussion}
\label{sec:conclusion}
\begin{revision}
In this paper, we propose SWIM, a behavior-informed framework that automatically identifies and weights useful state information for JEPA planning from offline trajectories. Rather than requiring human priors to preselect which candidate state variables are important for a task, SWIM uses recorded actions to learn their relative behavioral relevance, and we demonstrate the resulting weights in SCALE's state-distance calibration. Across five OGBench manipulation tasks and two planners, SWIM consistently improves over uniform state weighting (+0.9--1.0 percentage points, $p<0.05$), while the structured-distractor experiment shows that it downweights unrelated variables appended to the state. These results suggest that learning the state metric from behavior can reduce reliance on manually designed task-state representations, particularly when the relevant state information is difficult to identify a priori. SWIM depends on the state variables supplied during training and on the behaviors covered by the offline data, and we have demonstrated it through SCALE. Because the learned metric does not depend on the downstream objective, applying it to other state-supervised objectives, such as subspace alignment or physical-state grounding, is a natural next step. We also expect larger gains when the supplied state contains many more redundant or noisy coordinates, and future work should explore richer, state-dependent weighting mechanisms.
\end{revision}
\clearpage
\bibliographystyle{apalike}
\bibliography{refs}
\clearpage
\appendix
\section{Training Details and the Behavioral Objective}
\label{app:stage1}
\begin{revision}
In Appendices~\ref{app:stage1}--\ref{app:ablations}, gate values refer to
the original Stage-1 sigmoid gates, before the normalization used in the
planning analysis.

\paragraph{Data and Time Indexing.}
Stage~1 uses simulator states and actions from the same offline trajectories as Stage~2. Images, rewards, and success labels do not enter the behavioral predictor. Let $r$ denote a sampled raw-frame start and $\va_s^{\mathrm{raw}}$ the action at raw frame $s$. With frame skip $c=5$, the target action chunk is $\va_r=[(\va_r^{\mathrm{raw}})^\top,\ldots,(\va_{r+4}^{\mathrm{raw}})^\top]^\top$. A behavioral horizon $h\in\{1,2,4,8\}$ spans $ch$ raw frames. Starting frames are sampled uniformly over all raw frames whose $r+45$ frame lies in the same episode, irrespective of the sampled $h$; the fixed validity filter uses the maximum offset $8\cdot5+5=45$ raw frames. Horizon $h$ is distinct from the planner's rollout horizon $H$.

\paragraph{State and Action Standardization.}
Stage~1 and Stage~2 compute separate per-coordinate state means and standard deviations, each once over all frames of the offline dataset, before the train/validation split. In Stage~1, standardized state coordinates use a standard-deviation floor of $10^{-3}$ and are clipped to $[-10,10]$; Stage~2 does not clip standardized states. Stage~1 displacements are direct differences of its once-standardized states, with no additional normalization. Stage~1 action targets are standardized per coordinate and clipped to $[-10,10]$, and the Gaussian negative log-likelihood is evaluated in this standardized action space.

\paragraph{Predictor and Initialization.}
The Gaussian predictor is an MLP with two hidden layers of 1024 units, each followed by layer normalization (LayerNorm) and a Gaussian error linear unit (GELU) activation. Its outputs are the mean and diagonal log-standard-deviation of the action chunk. The latter is clamped to $[-5,2]$. Gate logits are initialized at $\alpha_k=2$, giving $g_k\approx0.88$. All candidate coordinates start with positive weight; no hard threshold is applied to the learned vector.

\paragraph{In-Batch Discrimination.}
The mismatched later states $\bar{\vq}$ in Equations~\eqref{eq:swim_losses} and~\eqref{eq:infonce} are taken from other examples in the same batch, each at its own sampled horizon, while the current state $\vq_t$ and action $\va_t$ stay fixed. The examples in a batch are ordered cyclically. For the hinge loss, the mismatched later state of each example is that of the next example in this order. For InfoNCE, the $N=255$ mismatched later states $\bar{\vq}_1,\ldots,\bar{\vq}_N$ are those of the next $N$ examples, with batch size $B=1024$ and temperature $\tau=1$. Thus, the denominator contains the observed displacement and 255 alternatives, not every example in the batch. Neither construction verifies that a mismatch is unreachable or behaviorally incompatible, so correlated or similar examples can weaken the contrast.

\paragraph{Regularization Settings.}
The reported grid uses $\lVert\vg\rVert_1$ with $\lambda\in\{0.01,0.03,0.05,0.10,1.0\}$ for both discrimination objectives. The reference hinge configuration uses margin $\kappa=1$ and $\lambda=0.03$; InfoNCE uses temperature $\tau=1$ and $\lambda=0.10$.
All Stage-1 runs use $\gamma=1.0$. Optimization uses Adam with learning rate $10^{-3}$ for 8000 steps, gradient clipping at 1.0, and a fixed gate-training seed. The hinge batch size is 4096 and the InfoNCE batch size is 1024.

\paragraph{Stage-2 Interface.}
The learned gates are detached. The state distances are computed after applying $\sqrt{\vg^\ast}$ to standardized coordinates, with no subsequent coordinate-wise restandardization. Stage~2 uses $\lambda_{\mathrm{corr}}=0.1$ and $\lambda_{\mathrm{sig}}=0.09$. Each batch samples 4096 frame pairs: half within trajectory segments and half across episodes, with rejection sampling ensuring that the two episodes in each cross-episode pair are different. The latent encoder and predictor, prediction and SIGReg terms, optimizer, data, and number of epochs match the \scaleall{} arm for the task; the optional state-regression head is disabled.
\end{revision}

\section{Metric Transfer and Interpretation}
\label{app:metric}
\begin{revision}
\paragraph{Preprocessing Order.}
If $q_k$ has unit variance, $\sqrt{g_k^\ast}q_k$ has variance $g_k^\ast$, so a second coordinate-wise standardization would divide by $\sqrt{g_k^\ast}$ and restore uniform weighting. The weights must therefore enter after state standardization.

\paragraph{Scale of the Gates.}
Multiplying $\vg$ by $b\in(0,1)$ and dividing the predictor's first-layer weights on the gated displacement by $b$ leaves every score $s_t(\dq)$ in Equation~\eqref{eq:swim_input} unchanged while reducing the gate penalty, so Stage~1 does not identify an absolute scale of the gates. This does not affect Stage~2, because Pearson correlation is invariant to a common positive rescaling of the state distances; only the relative weights matter.
\end{revision}

\section{State-Gate Values Across Tasks}
\label{app:gate_values}
\begin{revision}
The following tables report Stage-1 gate values for Cube-Double, Cube-Triple, Cube-Quadruple, and Scene. Each table includes the full candidate state vector for $\ell_1$-regularized Hinge and InfoNCE models at four sparsity coefficients. These tables report continuous gate values without thresholding.

Feature names in these tables identify coordinates in the supplied state representation.
The prefixes \texttt{b0}, \texttt{b1}, and so on identify cube slots,
with suffixes \texttt{x}, \texttt{y}, and \texttt{z} denoting position
coordinates; \texttt{eff} denotes end-effector position.
The \texttt{sin\_j} and \texttt{cos\_j} fields encode joint angles, and
\texttt{grip} fields describe the gripper.
The \texttt{sin2psi}, \texttt{cos2psi}, \texttt{sin4th}, and
\texttt{cos4th} fields are trigonometric orientation features.
The Scene fields \texttt{drawer}, \texttt{window}, and \texttt{btn}
refer to the corresponding objects and buttons.

For the multi-cube tasks, corresponding coordinates of different cube slots often receive closely matched gates. For example, at $\lambda=0.01$, InfoNCE assigns \texttt{b0\_y}, \texttt{b1\_y}, and \texttt{b2\_y} gates of 0.95 on Cube-Triple and all four corresponding \texttt{b0\_y}--\texttt{b3\_y} coordinates gates of 0.95 on Cube-Quadruple; their corresponding $x$ and $z$ coordinates are also closely grouped. This suggests that the learned weighting reflects shared coordinate roles across repeated objects, rather than arbitrary slot-specific preferences, while not implying exact equality at every coefficient.
\end{revision}

\begingroup
\normalcolor
\scriptsize
\setlength{\tabcolsep}{3pt}
\setlength{\LTcapwidth}{\textwidth}
\begin{longtable}{lrrrrrrrr}
\caption{\rev{\textbf{Stage-1 gate values on Cube-Double (27 dimensions).} Values are continuous gates learned with $\ell_1$ regularization and the two discrimination objectives across the displayed $\lambda$ values.}}
\label{tab:gates_cube_double}\\
    \toprule
    \textbf{Feature} & \multicolumn{4}{c}{\textbf{$\ell_1$ + Hinge}} & \multicolumn{4}{c}{\textbf{$\ell_1$ + InfoNCE}} \\
     & $\lambda=0.01$ & $0.03$ & $0.05$ & $0.10$ & $\lambda=0.01$ & $0.03$ & $0.05$ & $0.10$ \\
    \midrule
\endfirsthead
\multicolumn{9}{c}{\tablename~\thetable{} (continued)}\\
    \toprule
    \textbf{Feature} & \multicolumn{4}{c}{\textbf{$\ell_1$ + Hinge}} & \multicolumn{4}{c}{\textbf{$\ell_1$ + InfoNCE}} \\
     & $\lambda=0.01$ & $0.03$ & $0.05$ & $0.10$ & $\lambda=0.01$ & $0.03$ & $0.05$ & $0.10$ \\
    \midrule
\endhead
\midrule
\multicolumn{9}{r}{Continued on next page}\\
\endfoot
\bottomrule
\endlastfoot
    \texttt{b0\_y} & 0.94 & 0.87 & 0.76 & 0.51 & 0.95 & 0.91 & 0.86 & 0.66 \\
    \texttt{b1\_y} & 0.94 & 0.86 & 0.75 & 0.50 & 0.95 & 0.92 & 0.87 & 0.67 \\
    \texttt{b0\_x} & 0.92 & 0.79 & 0.63 & 0.39 & 0.93 & 0.88 & 0.79 & 0.55 \\
    \texttt{b1\_x} & 0.91 & 0.78 & 0.62 & 0.39 & 0.93 & 0.88 & 0.79 & 0.56 \\
    \texttt{eff\_z} & 0.88 & 0.84 & 0.80 & 0.68 & 0.88 & 0.85 & 0.81 & 0.71 \\
    \texttt{eff\_y} & 0.88 & 0.80 & 0.70 & 0.49 & 0.89 & 0.82 & 0.75 & 0.56 \\
    \texttt{cos\_j3} & 0.87 & 0.84 & 0.79 & 0.67 & 0.87 & 0.85 & 0.81 & 0.72 \\
    \texttt{grip\_open} & 0.84 & 0.77 & 0.70 & 0.53 & 0.84 & 0.78 & 0.71 & 0.55 \\
    \texttt{sin\_j3} & 0.84 & 0.76 & 0.65 & 0.45 & 0.84 & 0.78 & 0.71 & 0.53 \\
    \texttt{b1\_z} & 0.82 & 0.74 & 0.64 & 0.46 & 0.84 & 0.76 & 0.68 & 0.51 \\
    \texttt{b0\_z} & 0.82 & 0.74 & 0.64 & 0.47 & 0.84 & 0.76 & 0.68 & 0.51 \\
    \texttt{eff\_x} & 0.82 & 0.67 & 0.52 & 0.32 & 0.83 & 0.73 & 0.61 & 0.40 \\
    \texttt{cos\_j1} & 0.80 & 0.66 & 0.51 & 0.29 & 0.80 & 0.69 & 0.57 & 0.35 \\
    \texttt{cos\_j0} & 0.78 & 0.60 & 0.45 & 0.25 & 0.79 & 0.65 & 0.51 & 0.29 \\
    \texttt{cos2psi} & 0.75 & 0.58 & 0.45 & 0.29 & 0.79 & 0.66 & 0.55 & 0.37 \\
    \texttt{sin2psi} & 0.73 & 0.56 & 0.44 & 0.28 & 0.77 & 0.64 & 0.53 & 0.35 \\
    \texttt{sin\_j2} & 0.67 & 0.45 & 0.31 & 0.16 & 0.64 & 0.43 & 0.30 & 0.15 \\
    \texttt{cos\_j2} & 0.67 & 0.45 & 0.30 & 0.14 & 0.66 & 0.47 & 0.33 & 0.16 \\
    \texttt{sin\_j0} & 0.64 & 0.44 & 0.31 & 0.18 & 0.67 & 0.50 & 0.37 & 0.21 \\
    \texttt{grip\_contact} & 0.63 & 0.47 & 0.36 & 0.20 & 0.63 & 0.48 & 0.38 & 0.23 \\
    \texttt{b0\_cos4th} & 0.57 & 0.36 & 0.26 & 0.17 & 0.72 & 0.52 & 0.40 & 0.25 \\
    \texttt{b0\_sin4th} & 0.56 & 0.35 & 0.26 & 0.16 & 0.72 & 0.53 & 0.40 & 0.25 \\
    \texttt{b1\_sin4th} & 0.56 & 0.35 & 0.26 & 0.16 & 0.71 & 0.52 & 0.39 & 0.24 \\
    \texttt{b1\_cos4th} & 0.55 & 0.34 & 0.25 & 0.16 & 0.71 & 0.51 & 0.39 & 0.24 \\
    \texttt{sin\_j1} & 0.49 & 0.35 & 0.26 & 0.17 & 0.51 & 0.38 & 0.30 & 0.20 \\
    \texttt{cos\_j4} & 0.25 & 0.10 & 0.06 & 0.04 & 0.27 & 0.10 & 0.06 & 0.03 \\
    \texttt{sin\_j4} & 0.06 & 0.03 & 0.03 & 0.02 & 0.09 & 0.03 & 0.03 & 0.02 \\
\end{longtable}
\endgroup
\begingroup
\normalcolor
\scriptsize
\setlength{\tabcolsep}{3pt}
\setlength{\LTcapwidth}{\textwidth}
\begin{longtable}{lrrrrrrrr}
\caption{\rev{\textbf{Stage-1 gate values on Cube-Triple (32 dimensions).} Values are continuous gates learned with $\ell_1$ regularization and the two discrimination objectives across the displayed $\lambda$ values.}}
\label{tab:gates_cube_triple}\\
    \toprule
    \textbf{Feature} & \multicolumn{4}{c}{\textbf{$\ell_1$ + Hinge}} & \multicolumn{4}{c}{\textbf{$\ell_1$ + InfoNCE}} \\
     & $\lambda=0.01$ & $0.03$ & $0.05$ & $0.10$ & $\lambda=0.01$ & $0.03$ & $0.05$ & $0.10$ \\
    \midrule
\endfirsthead
\multicolumn{9}{c}{\tablename~\thetable{} (continued)}\\
    \toprule
    \textbf{Feature} & \multicolumn{4}{c}{\textbf{$\ell_1$ + Hinge}} & \multicolumn{4}{c}{\textbf{$\ell_1$ + InfoNCE}} \\
     & $\lambda=0.01$ & $0.03$ & $0.05$ & $0.10$ & $\lambda=0.01$ & $0.03$ & $0.05$ & $0.10$ \\
    \midrule
\endhead
\midrule
\multicolumn{9}{r}{Continued on next page}\\
\endfoot
\bottomrule
\endlastfoot
    \texttt{b0\_y} & 0.93 & 0.82 & 0.67 & 0.41 & 0.95 & 0.90 & 0.82 & 0.58 \\
    \texttt{b1\_y} & 0.93 & 0.82 & 0.66 & 0.40 & 0.95 & 0.90 & 0.82 & 0.57 \\
    \texttt{b2\_y} & 0.93 & 0.81 & 0.65 & 0.40 & 0.95 & 0.90 & 0.82 & 0.57 \\
    \texttt{b2\_x} & 0.89 & 0.69 & 0.50 & 0.30 & 0.93 & 0.85 & 0.73 & 0.46 \\
    \texttt{b0\_x} & 0.89 & 0.68 & 0.50 & 0.30 & 0.94 & 0.86 & 0.74 & 0.47 \\
    \texttt{b1\_x} & 0.89 & 0.68 & 0.49 & 0.29 & 0.94 & 0.86 & 0.73 & 0.47 \\
    \texttt{eff\_y} & 0.88 & 0.80 & 0.69 & 0.48 & 0.89 & 0.82 & 0.73 & 0.54 \\
    \texttt{eff\_z} & 0.88 & 0.84 & 0.79 & 0.65 & 0.87 & 0.84 & 0.80 & 0.69 \\
    \texttt{cos\_j3} & 0.88 & 0.83 & 0.77 & 0.61 & 0.87 & 0.84 & 0.81 & 0.70 \\
    \texttt{grip\_open} & 0.84 & 0.76 & 0.67 & 0.49 & 0.84 & 0.76 & 0.68 & 0.50 \\
    \texttt{eff\_x} & 0.83 & 0.68 & 0.53 & 0.32 & 0.85 & 0.74 & 0.62 & 0.40 \\
    \texttt{sin\_j3} & 0.83 & 0.73 & 0.61 & 0.39 & 0.84 & 0.77 & 0.69 & 0.49 \\
    \texttt{cos\_j1} & 0.81 & 0.66 & 0.52 & 0.30 & 0.80 & 0.68 & 0.55 & 0.33 \\
    \texttt{b1\_z} & 0.79 & 0.66 & 0.54 & 0.36 & 0.79 & 0.67 & 0.56 & 0.38 \\
    \texttt{b0\_z} & 0.78 & 0.64 & 0.53 & 0.35 & 0.80 & 0.69 & 0.58 & 0.40 \\
    \texttt{cos\_j0} & 0.77 & 0.57 & 0.40 & 0.21 & 0.81 & 0.67 & 0.53 & 0.30 \\
    \texttt{b2\_z} & 0.77 & 0.63 & 0.52 & 0.34 & 0.79 & 0.68 & 0.57 & 0.39 \\
    \texttt{cos2psi} & 0.75 & 0.57 & 0.44 & 0.28 & 0.81 & 0.68 & 0.56 & 0.37 \\
    \texttt{sin2psi} & 0.74 & 0.56 & 0.43 & 0.27 & 0.81 & 0.69 & 0.56 & 0.37 \\
    \texttt{cos\_j2} & 0.69 & 0.45 & 0.30 & 0.14 & 0.66 & 0.44 & 0.29 & 0.14 \\
    \texttt{sin\_j2} & 0.67 & 0.43 & 0.29 & 0.14 & 0.62 & 0.38 & 0.25 & 0.12 \\
    \texttt{sin\_j0} & 0.65 & 0.43 & 0.31 & 0.17 & 0.68 & 0.49 & 0.35 & 0.20 \\
    \texttt{grip\_contact} & 0.62 & 0.43 & 0.32 & 0.18 & 0.62 & 0.46 & 0.34 & 0.20 \\
    \texttt{b2\_sin4th} & 0.57 & 0.34 & 0.24 & 0.15 & 0.75 & 0.54 & 0.40 & 0.24 \\
    \texttt{b1\_sin4th} & 0.56 & 0.34 & 0.24 & 0.15 & 0.74 & 0.53 & 0.39 & 0.24 \\
    \texttt{b0\_cos4th} & 0.56 & 0.33 & 0.24 & 0.15 & 0.74 & 0.53 & 0.39 & 0.23 \\
    \texttt{b0\_sin4th} & 0.56 & 0.33 & 0.24 & 0.15 & 0.73 & 0.51 & 0.38 & 0.23 \\
    \texttt{b2\_cos4th} & 0.55 & 0.33 & 0.24 & 0.15 & 0.71 & 0.49 & 0.36 & 0.21 \\
    \texttt{b1\_cos4th} & 0.54 & 0.32 & 0.23 & 0.15 & 0.74 & 0.53 & 0.39 & 0.23 \\
    \texttt{sin\_j1} & 0.48 & 0.33 & 0.25 & 0.16 & 0.52 & 0.38 & 0.29 & 0.19 \\
    \texttt{cos\_j4} & 0.21 & 0.09 & 0.06 & 0.04 & 0.29 & 0.12 & 0.07 & 0.04 \\
    \texttt{sin\_j4} & 0.12 & 0.04 & 0.03 & 0.02 & 0.07 & 0.03 & 0.02 & 0.02 \\
\end{longtable}
\endgroup
\begingroup
\normalcolor
\scriptsize
\setlength{\tabcolsep}{3pt}
\setlength{\LTcapwidth}{\textwidth}
\begin{longtable}{lrrrrrrrr}
\caption{\rev{\textbf{Stage-1 gate values on Cube-Quadruple (37 dimensions).} Values are continuous gates learned with $\ell_1$ regularization and the two discrimination objectives across the displayed $\lambda$ values.}}
\label{tab:gates_cube_quadruple}\\
    \toprule
    \textbf{Feature} & \multicolumn{4}{c}{\textbf{$\ell_1$ + Hinge}} & \multicolumn{4}{c}{\textbf{$\ell_1$ + InfoNCE}} \\
     & $\lambda=0.01$ & $0.03$ & $0.05$ & $0.10$ & $\lambda=0.01$ & $0.03$ & $0.05$ & $0.10$ \\
    \midrule
\endfirsthead
\multicolumn{9}{c}{\tablename~\thetable{} (continued)}\\
    \toprule
    \textbf{Feature} & \multicolumn{4}{c}{\textbf{$\ell_1$ + Hinge}} & \multicolumn{4}{c}{\textbf{$\ell_1$ + InfoNCE}} \\
     & $\lambda=0.01$ & $0.03$ & $0.05$ & $0.10$ & $\lambda=0.01$ & $0.03$ & $0.05$ & $0.10$ \\
    \midrule
\endhead
\midrule
\multicolumn{9}{r}{Continued on next page}\\
\endfoot
\bottomrule
\endlastfoot
    \texttt{b3\_y} & 0.93 & 0.77 & 0.58 & 0.34 & 0.95 & 0.89 & 0.79 & 0.52 \\
    \texttt{b2\_y} & 0.92 & 0.76 & 0.57 & 0.33 & 0.95 & 0.89 & 0.78 & 0.51 \\
    \texttt{b0\_y} & 0.92 & 0.75 & 0.57 & 0.33 & 0.95 & 0.88 & 0.77 & 0.50 \\
    \texttt{b1\_y} & 0.92 & 0.75 & 0.57 & 0.33 & 0.95 & 0.88 & 0.77 & 0.50 \\
    \texttt{eff\_y} & 0.89 & 0.80 & 0.69 & 0.48 & 0.89 & 0.82 & 0.74 & 0.53 \\
    \texttt{b0\_x} & 0.88 & 0.62 & 0.43 & 0.25 & 0.93 & 0.82 & 0.67 & 0.41 \\
    \texttt{b2\_x} & 0.88 & 0.62 & 0.44 & 0.26 & 0.93 & 0.83 & 0.68 & 0.41 \\
    \texttt{b1\_x} & 0.88 & 0.62 & 0.44 & 0.26 & 0.93 & 0.84 & 0.69 & 0.42 \\
    \texttt{b3\_x} & 0.88 & 0.62 & 0.44 & 0.26 & 0.93 & 0.83 & 0.69 & 0.42 \\
    \texttt{cos\_j3} & 0.87 & 0.82 & 0.76 & 0.59 & 0.88 & 0.84 & 0.80 & 0.69 \\
    \texttt{eff\_z} & 0.87 & 0.82 & 0.77 & 0.62 & 0.88 & 0.84 & 0.81 & 0.70 \\
    \texttt{sin\_j3} & 0.84 & 0.74 & 0.64 & 0.42 & 0.85 & 0.79 & 0.71 & 0.53 \\
    \texttt{eff\_x} & 0.83 & 0.65 & 0.50 & 0.30 & 0.85 & 0.73 & 0.59 & 0.38 \\
    \texttt{grip\_open} & 0.82 & 0.70 & 0.58 & 0.39 & 0.85 & 0.77 & 0.68 & 0.50 \\
    \texttt{cos\_j1} & 0.81 & 0.68 & 0.54 & 0.32 & 0.81 & 0.68 & 0.54 & 0.31 \\
    \texttt{cos\_j0} & 0.78 & 0.56 & 0.39 & 0.21 & 0.80 & 0.63 & 0.47 & 0.25 \\
    \texttt{cos2psi} & 0.78 & 0.60 & 0.46 & 0.29 & 0.83 & 0.70 & 0.57 & 0.37 \\
    \texttt{sin2psi} & 0.78 & 0.59 & 0.45 & 0.29 & 0.83 & 0.69 & 0.56 & 0.37 \\
    \texttt{b2\_z} & 0.74 & 0.57 & 0.45 & 0.28 & 0.76 & 0.63 & 0.51 & 0.33 \\
    \texttt{b1\_z} & 0.72 & 0.55 & 0.43 & 0.27 & 0.76 & 0.63 & 0.52 & 0.34 \\
    \texttt{b3\_z} & 0.72 & 0.55 & 0.42 & 0.27 & 0.77 & 0.65 & 0.53 & 0.35 \\
    \texttt{b0\_z} & 0.71 & 0.54 & 0.42 & 0.26 & 0.76 & 0.63 & 0.52 & 0.34 \\
    \texttt{cos\_j2} & 0.67 & 0.41 & 0.27 & 0.13 & 0.65 & 0.42 & 0.27 & 0.12 \\
    \texttt{sin\_j0} & 0.66 & 0.44 & 0.31 & 0.18 & 0.69 & 0.48 & 0.35 & 0.19 \\
    \texttt{sin\_j2} & 0.66 & 0.40 & 0.26 & 0.13 & 0.69 & 0.46 & 0.31 & 0.15 \\
    \texttt{grip\_contact} & 0.57 & 0.37 & 0.26 & 0.15 & 0.65 & 0.48 & 0.36 & 0.21 \\
    \texttt{b1\_cos4th} & 0.53 & 0.30 & 0.21 & 0.13 & 0.72 & 0.49 & 0.35 & 0.21 \\
    \texttt{b3\_sin4th} & 0.52 & 0.29 & 0.21 & 0.12 & 0.73 & 0.50 & 0.36 & 0.22 \\
    \texttt{b0\_sin4th} & 0.52 & 0.29 & 0.21 & 0.13 & 0.74 & 0.51 & 0.37 & 0.22 \\
    \texttt{b3\_cos4th} & 0.51 & 0.29 & 0.21 & 0.13 & 0.72 & 0.48 & 0.35 & 0.20 \\
    \texttt{b2\_sin4th} & 0.51 & 0.28 & 0.20 & 0.12 & 0.73 & 0.50 & 0.36 & 0.21 \\
    \texttt{b1\_sin4th} & 0.50 & 0.28 & 0.20 & 0.12 & 0.73 & 0.51 & 0.37 & 0.22 \\
    \texttt{b0\_cos4th} & 0.50 & 0.28 & 0.21 & 0.13 & 0.70 & 0.46 & 0.32 & 0.19 \\
    \texttt{b2\_cos4th} & 0.48 & 0.27 & 0.20 & 0.12 & 0.73 & 0.50 & 0.36 & 0.22 \\
    \texttt{sin\_j1} & 0.47 & 0.31 & 0.23 & 0.14 & 0.53 & 0.38 & 0.29 & 0.18 \\
    \texttt{cos\_j4} & 0.19 & 0.08 & 0.06 & 0.04 & 0.32 & 0.12 & 0.07 & 0.04 \\
    \texttt{sin\_j4} & 0.04 & 0.03 & 0.03 & 0.02 & 0.11 & 0.04 & 0.03 & 0.02 \\
\end{longtable}
\endgroup
\begingroup
\normalcolor
\scriptsize
\setlength{\tabcolsep}{3pt}
\setlength{\LTcapwidth}{\textwidth}
\begin{longtable}{lrrrrrrrr}
\caption{\rev{\textbf{Stage-1 gate values on Scene (26 dimensions).} Values are continuous gates learned with $\ell_1$ regularization and the two discrimination objectives across the displayed $\lambda$ values.}}
\label{tab:gates_scene}\\
    \toprule
    \textbf{Feature} & \multicolumn{4}{c}{\textbf{$\ell_1$ + Hinge}} & \multicolumn{4}{c}{\textbf{$\ell_1$ + InfoNCE}} \\
     & $\lambda=0.01$ & $0.03$ & $0.05$ & $0.10$ & $\lambda=0.01$ & $0.03$ & $0.05$ & $0.10$ \\
    \midrule
\endfirsthead
\multicolumn{9}{c}{\tablename~\thetable{} (continued)}\\
    \toprule
    \textbf{Feature} & \multicolumn{4}{c}{\textbf{$\ell_1$ + Hinge}} & \multicolumn{4}{c}{\textbf{$\ell_1$ + InfoNCE}} \\
     & $\lambda=0.01$ & $0.03$ & $0.05$ & $0.10$ & $\lambda=0.01$ & $0.03$ & $0.05$ & $0.10$ \\
    \midrule
\endhead
\midrule
\multicolumn{9}{r}{Continued on next page}\\
\endfoot
\bottomrule
\endlastfoot
    \texttt{drawer} & 0.89 & 0.78 & 0.66 & 0.44 & 0.96 & 0.92 & 0.86 & 0.68 \\
    \texttt{eff\_z} & 0.88 & 0.82 & 0.76 & 0.61 & 0.87 & 0.83 & 0.78 & 0.63 \\
    \texttt{cos\_j3} & 0.87 & 0.83 & 0.77 & 0.60 & 0.87 & 0.83 & 0.78 & 0.64 \\
    \texttt{window} & 0.85 & 0.73 & 0.60 & 0.40 & 0.92 & 0.87 & 0.79 & 0.60 \\
    \texttt{sin\_j2} & 0.81 & 0.63 & 0.48 & 0.27 & 0.80 & 0.66 & 0.52 & 0.29 \\
    \texttt{eff\_y} & 0.80 & 0.63 & 0.49 & 0.31 & 0.80 & 0.66 & 0.53 & 0.35 \\
    \texttt{btn0} & 0.80 & 0.62 & 0.48 & 0.29 & 0.83 & 0.71 & 0.60 & 0.40 \\
    \texttt{grip\_open} & 0.80 & 0.70 & 0.62 & 0.44 & 0.79 & 0.70 & 0.61 & 0.45 \\
    \texttt{cos2psi} & 0.79 & 0.63 & 0.50 & 0.33 & 0.82 & 0.70 & 0.59 & 0.40 \\
    \texttt{eff\_x} & 0.79 & 0.64 & 0.49 & 0.28 & 0.78 & 0.65 & 0.53 & 0.32 \\
    \texttt{sin\_j3} & 0.79 & 0.62 & 0.48 & 0.28 & 0.80 & 0.67 & 0.55 & 0.34 \\
    \texttt{sin\_j0} & 0.79 & 0.62 & 0.48 & 0.29 & 0.80 & 0.67 & 0.54 & 0.34 \\
    \texttt{btn1} & 0.79 & 0.60 & 0.46 & 0.28 & 0.82 & 0.71 & 0.59 & 0.39 \\
    \texttt{cos\_j0} & 0.78 & 0.60 & 0.45 & 0.27 & 0.79 & 0.66 & 0.54 & 0.34 \\
    \texttt{cos\_j1} & 0.78 & 0.64 & 0.50 & 0.29 & 0.77 & 0.65 & 0.54 & 0.33 \\
    \texttt{b0\_y} & 0.75 & 0.46 & 0.31 & 0.18 & 0.78 & 0.50 & 0.35 & 0.20 \\
    \texttt{cos\_j2} & 0.69 & 0.46 & 0.31 & 0.14 & 0.69 & 0.50 & 0.35 & 0.17 \\
    \texttt{sin2psi} & 0.67 & 0.50 & 0.39 & 0.25 & 0.72 & 0.59 & 0.49 & 0.33 \\
    \texttt{sin\_j1} & 0.64 & 0.48 & 0.38 & 0.24 & 0.63 & 0.50 & 0.41 & 0.27 \\
    \texttt{sin\_j4} & 0.55 & 0.34 & 0.25 & 0.15 & 0.63 & 0.44 & 0.33 & 0.20 \\
    \texttt{cos\_j4} & 0.50 & 0.31 & 0.23 & 0.13 & 0.57 & 0.39 & 0.29 & 0.18 \\
    \texttt{b0\_x} & 0.49 & 0.26 & 0.17 & 0.10 & 0.56 & 0.31 & 0.21 & 0.12 \\
    \texttt{b0\_z} & 0.47 & 0.27 & 0.19 & 0.11 & 0.49 & 0.30 & 0.22 & 0.13 \\
    \texttt{grip\_contact} & 0.27 & 0.17 & 0.13 & 0.08 & 0.32 & 0.21 & 0.16 & 0.10 \\
    \texttt{b0\_sin4th} & 0.26 & 0.13 & 0.08 & 0.05 & 0.34 & 0.17 & 0.11 & 0.06 \\
    \texttt{b0\_cos4th} & 0.19 & 0.09 & 0.06 & 0.04 & 0.27 & 0.13 & 0.09 & 0.05 \\
\end{longtable}
\endgroup

\section{Structured-Distractor Diagnostics}
\label{app:contamination}
\begin{revision}
The contamination experiment appends 26 coordinates from unrelated Scene trajectories to Cube's 22-dimensional state. Gate stability compares the weights assigned to the \emph{same original Cube coordinates} under clean and augmented training. For $k=1,\ldots,22$, let $g_k^{\rm clean}$ and $g_k^{\rm aug}$ denote these two gate values, and define $\Delta g_k=g_k^{\rm aug}-g_k^{\rm clean}$. The mean and maximum of $|\Delta g_k|$ are taken over these 22 coordinates. We denote Spearman rank correlation by $\rho_{\mathrm S}$ and Pearson correlation by $\rho_{\mathrm P}$. The reporting sets are $\mathcal S_{\rm clean}=\{k\in\{1,\ldots,22\}:g_k^{\rm clean}\ge0.5\}$ and $\mathcal S_{\rm aug}=\{k\in\{1,\ldots,22\}:g_k^{\rm aug}\ge0.5\}$. Jaccard overlap is $|\mathcal S_{\rm clean}\cap\mathcal S_{\rm aug}|/|\mathcal S_{\rm clean}\cup\mathcal S_{\rm aug}|$. The gates remain continuous in the planner's training target.
\end{revision}
\begin{table}[!htbp]
  \normalcolor
  \centering
  \small
  \begin{tabular}{ccccccc}
    \toprule
    $\lambda$ & \textbf{Spearman $\rho_{\mathrm S}$} & \textbf{Pearson $\rho_{\mathrm P}$} & \textbf{mean $|\Delta g_k|$} & \textbf{max $|\Delta g_k|$} & \textbf{Counts} & \textbf{Jaccard} \\
    \midrule
    0.01 & 0.959 & 0.966 & 0.040 & 0.08 & 21 / 22 & 0.95 \\
    0.03 & 0.981 & 0.974 & 0.037 & 0.09 & 20 / 20 & 1.00 \\
    0.05 & 0.977 & 0.982 & 0.029 & 0.08 & 13 / 15 & 0.87 \\
    0.10 & \textbf{0.995} & 0.983 & \textbf{0.023} & 0.14 & 5 / 5 & \textbf{1.00} \\
    \bottomrule
  \end{tabular}
  \caption{\rev{\textbf{Gate stability under the Scene-coordinate augmentation.} The comparison uses the original 22 Cube coordinates, with InfoNCE and $\lVert\vg\rVert_1$ regularization. Spearman $\rho_{\mathrm S}$ and Pearson $\rho_{\mathrm P}$ measure rank and linear association, respectively. Counts give clean / augmented numbers of coordinates at or above the reporting threshold.}}
  \label{tab:stability}
\end{table}
\begin{table}[!htbp]
  \normalcolor
  \centering
  \small
  \begin{tabular}{llcc}
    \toprule
    \textbf{Alignment target} & \textbf{State} & \textbf{CEM} & \textbf{iCEM} \\
    \midrule
    none (\lewm)                          & --                     & 67.7 & 70.8 \\
    \scaleall                             & clean, 22-d            & \textbf{75.5} & 76.3 \\
    \scaleall                             & contaminated, 48-d     & 74.2 & 75.5 \\
    \method{} ($\lambda$=.05)             & contaminated, 48-d     & 75.2 & 76.7 \\
    \method{} ($\lambda$=.10)             & contaminated, 48-d     & 74.5 & \textbf{76.8} \\
    \bottomrule
  \end{tabular}
  \caption{\rev{\textbf{Downstream planning under state contamination on Cube.} CEM/iCEM success percentages from a single training seed. The reference is clean \scaleall; bold marks the largest value in each planner column.}}
  \label{tab:contam_plan}
\end{table}

\section{Experimental Details and Planning Protocol}
\label{app:planning}
\begin{revision}
\paragraph{Tasks and Evaluation.}
The state dimensions for Cube, Cube-Double, Cube-Triple, Cube-Quadruple, and Scene are 22, 27, 32, 37, and 26, respectively. Each \lewm{}-based method is trained with five seeds. Each task, configuration, and planner is evaluated for 600 episodes per training seed, distributed across six evaluation seeds with 100 episodes per seed. Table~\ref{tab:icem} reports iCEM planning success in the format of Table~\ref{tab:main}.

\paragraph{CEM Planning Protocol.}
Following the CEM setup in LeWM~\citep{maes2026leworldmodel}, each planning iteration evaluates 300 candidate action sequences and retains 30 elite sequences. CEM uses 30 iterations for Push-T and 10 iterations for other environments; the five OGBench tasks evaluated in this work therefore use 10 iterations. The planning horizon is five decision steps with frame skip 5 (25 environment steps). Model predictive control (MPC) executes the complete five-step plan before replanning. Each evaluation episode allows at most 50 environment steps. The goal is sampled from the same offline trajectory at the state 25 environment steps after the initial state.
\end{revision}

\begin{table}[!htbp]
  \normalcolor
  \centering
  \footnotesize
  \setlength{\tabcolsep}{3.5pt}
  \begin{tabular}{lrrrrrr}
    \toprule
    \textbf{Task} & \textbf{\lewm} & \textbf{\scaleall} & \textbf{SWIM-H$^\star$} & \textbf{SWIM-N$^\star$} & $\Delta_{\text{\scaleall}}$ & $\Delta_{\text{\lewm}}$ \\
    \midrule
    Cube           & 67.8{\scriptsize$\pm$1.7} & \underline{76.8}{\scriptsize$\pm$0.5} & \textbf{77.1}{\scriptsize$\pm$0.6} {\scriptsize(.03)} & \underline{76.8}{\scriptsize$\pm$0.4} {\scriptsize(.03)} & +0.2 & +9.3 \\
    Cube-Double    & 18.0{\scriptsize$\pm$1.7} & 19.5{\scriptsize$\pm$1.8} & \underline{21.6}{\scriptsize$\pm$1.4} {\scriptsize(1.0)} & \textbf{22.6}{\scriptsize$\pm$1.9} {\scriptsize(.05)} & +3.1 & +4.6 \\
    Cube-Triple    & 35.3{\scriptsize$\pm$1.5} & 36.0{\scriptsize$\pm$1.5} & \underline{36.9}{\scriptsize$\pm$1.3} {\scriptsize(.1)} & \textbf{37.1}{\scriptsize$\pm$0.9} {\scriptsize(1.0)} & +1.1 & +1.8 \\
    Cube-Quadruple & 33.9{\scriptsize$\pm$1.5} & 34.7{\scriptsize$\pm$1.0} & \underline{35.2}{\scriptsize$\pm$1.3} {\scriptsize(1.0)} & \textbf{35.7}{\scriptsize$\pm$0.2} {\scriptsize(.05)} & +1.0 & +1.8 \\
    Scene          & 24.6{\scriptsize$\pm$1.8} & 26.3{\scriptsize$\pm$1.2} & \underline{26.7}{\scriptsize$\pm$1.1} {\scriptsize(.03)} & \textbf{27.1}{\scriptsize$\pm$1.5} {\scriptsize(.01)} & +0.7 & +2.5 \\
    \bottomrule
  \end{tabular}
  \caption{\rev{\textbf{iCEM planning success (\%) over five training seeds} (mean $\pm$ standard deviation). Notation as in Table~\ref{tab:main}.}}
  \label{tab:icem}
\end{table}

\section{Ablation Studies}
\label{app:ablations}
\begin{revision}
All planning results below use a single training seed and follow the planning protocol in Appendix~\ref{app:planning}. The full models in Table~\ref{tab:loss_ablation} and all settings in Table~\ref{tab:margin_temperature} use $\lambda=0.03$ for both discrimination objectives so that they share the same gate penalty. Entries with $\pm$ report mean and standard deviation across the six evaluation seeds.
\end{revision}

\subsection{Loss-Component Ablations}
\begin{table}[!htbp]
  \normalcolor
  \centering
  \small
  \begin{tabular}{lccc}
    \toprule
    \textbf{Variant} & \shortstack{\textbf{CEM success}\\\textbf{(\%)}} & \shortstack{\textbf{iCEM success}\\\textbf{(\%)}} & \shortstack{\textbf{Dimensions}\\\textbf{with $g_k\ge0.5$}} \\
    \midrule
    Full Hinge ($\kappa=1$) & $75.5\pm4.2$ & $78.0\pm3.3$ & 15/22 \\
    Full InfoNCE ($\tau=1$) & $77.0\pm6.6$ & $77.3\pm3.6$ & 20/22 \\
    No discrimination term & $76.2\pm5.3$ & $76.8\pm4.7$ & 13/22 \\
    No NLL & $75.7\pm6.2$ & $76.7\pm3.9$ & 0/22 \\
    \bottomrule
  \end{tabular}
  \caption{\rev{\textbf{Loss-component ablations.} Success rates are percentages. The gate-count threshold is used only for reporting; gates remain continuous. Both ablated objectives retain the gate penalty $\lambda\lVert\vg\rVert_1$.}}
  \label{tab:loss_ablation}
\end{table}
\begin{revision}
The four SWIM variants have similar success rates. Removing the NLL leaves no gate at or above the reporting threshold, whereas variants retaining the NLL leave 13--20 dimensions at or above threshold. For Hinge, removing the discrimination term reduces the number of retained dimensions from 15 to 13. These results suggest that the NLL supports nontrivial gates and that the discrimination term affects the learned gate profile.
\end{revision}

\subsection{Gate Penalty and Discrimination Objective}
The $\ell_1$ penalty is $\lVert\vg\rVert_1$; the squared $\ell_2$
penalty is $\lVert\vg\rVert_2^2=\sum_k g_k^2$.
The latter replaces the $\ell_1$ penalty in
Equation~\eqref{eq:swim_stage1} for the corresponding ablations.

\begin{table}[!htbp]
  \normalcolor
  \centering
  \small
  \setlength{\tabcolsep}{5pt}
  \begin{tabular}{llcc}
    \toprule
    \textbf{Penalty + objective} & $\lambda$ & \textbf{CEM (\%)} & \textbf{iCEM (\%)} \\
    \midrule
    \lewm{} baseline & -- & 67.7 & 70.8 \\
    SCALE (9 selected dimensions) & -- & 73.5 & 76.7 \\
    \midrule
    \multirow{4}{*}{$\ell_1$ + Hinge} & 0.01 & 76.0 & 77.5 \\
      & 0.03 & 75.5 & 78.0 \\
      & 0.05 & 75.2 & 75.5 \\
      & 0.10 & 76.2 & 75.2 \\
    \midrule
    \multirow{4}{*}{$\ell_2^2$ + Hinge} & 0.01 & 74.5 & 76.0 \\
      & 0.03 & 74.5 & 77.0 \\
      & 0.05 & 75.7 & 76.3 \\
      & 0.10 & 75.3 & 75.8 \\
    \midrule
    \multirow{4}{*}{$\ell_1$ + InfoNCE} & 0.01 & 74.8 & 76.2 \\
      & 0.03 & 77.0 & 77.3 \\
      & 0.05 & 76.0 & 78.0 \\
      & 0.10 & 75.7 & 76.8 \\
    \midrule
    \multirow{4}{*}{$\ell_2^2$ + InfoNCE} & 0.01 & 76.3 & 76.5 \\
      & 0.03 & 75.0 & 77.0 \\
      & 0.05 & 76.0 & 75.8 \\
      & 0.10 & 75.8 & 76.7 \\
    \bottomrule
  \end{tabular}
  \caption{\rev{\textbf{Regularization--objective comparison.} All entries are planning success percentages under the planning protocol in Appendix~\ref{app:planning}.}}
  \label{tab:penalty_objective}
\end{table}
\begin{revision}
On Cube, all 16 SWIM configurations exceed SCALE with nine manually selected state dimensions under CEM (74.5--77.0\% vs.\ 73.5\%) and are comparable under iCEM (75.2--78.0\% vs.\ 76.7\%), without requiring the manual selection. The Spearman rank correlation between $\ell_1$- and squared-$\ell_2$-regularized gates is 0.97 at $\lambda=0.01$ and 0.99 at $\lambda=0.1$. At $\lambda=0.3$ and 1.0, all squared-$\ell_2$ gates fall below 0.5 without approaching zero, so the squared penalty preserves graded weights but does not produce the same threshold-based feature selection as $\ell_1$. The $\ell_2^2$+Hinge, $\ell_1$+InfoNCE, and $\ell_2^2$+InfoNCE variants share the retained set $\{\texttt{eff\_z},\texttt{cos\_j3},\texttt{sin\_j3},\texttt{block\_z},\texttt{cos\_j1}\}$. $\ell_1$+Hinge is the exception, retaining \texttt{grip\_open} and \texttt{block\_y} among its selected dimensions.
\end{revision}

\subsection{Hinge-Margin and InfoNCE-Temperature Sensitivity}
\begin{table}[!htbp]
  \normalcolor
  \centering
  \small
  \begin{tabular}{lcc}
    \toprule
    \textbf{Setting} & \textbf{CEM success (\%)} & \textbf{iCEM success (\%)} \\
    \midrule
    Hinge, $\kappa=0.3$ & $74.7\pm4.3$ & $77.3\pm3.1$ \\
    Hinge, $\kappa=1$ (default) & $75.5\pm4.2$ & $78.0\pm3.3$ \\
    Hinge, $\kappa=3$ & $73.5\pm2.7$ & $77.0\pm3.6$ \\
    Hinge, $\kappa=10$ & $75.3\pm5.1$ & $76.8\pm4.0$ \\
    \midrule
    InfoNCE, $\tau=0.3$ & $75.0\pm5.6$ & $77.7\pm3.7$ \\
    InfoNCE, $\tau=1$ (default) & $77.0\pm6.6$ & $77.3\pm3.6$ \\
    InfoNCE, $\tau=3$ & $75.7\pm4.2$ & $77.5\pm3.9$ \\
    InfoNCE, $\tau=10$ & $72.8\pm4.0$ & $76.8\pm4.2$ \\
    InfoNCE, $\tau=30$ & $73.8\pm4.7$ & $76.2\pm4.2$ \\
    InfoNCE, $\tau=100$ & $75.5\pm4.3$ & $76.8\pm4.6$ \\
    InfoNCE, $\tau=300$ & $73.8\pm4.2$ & $77.0\pm4.6$ \\
    \bottomrule
  \end{tabular}
  \caption{\rev{\textbf{Margin and temperature sensitivity.} Values report planning success percentages as mean $\pm$ standard deviation over six evaluation seeds.}}
  \label{tab:margin_temperature}
\end{table}
\begin{revision}
Hinge success changes modestly over the tested margins. For InfoNCE, $\tau=1$ has the highest mean across CEM and iCEM among the tested temperatures; larger temperatures do not improve CEM success. We therefore use $\kappa=1$ and $\tau=1$ as defaults.
\end{revision}

\subsection{Data Efficiency of Gate Learning}
\begin{revision}
For each subset size, we sample three subsets of complete trajectories from the 2,000-trajectory dataset. Selected dimensions are coordinates with $g_k\ge0.5$; their reported count is averaged over the three subsets. We use the gates learned from all 2,000 trajectories as the reference for the Jaccard and Spearman comparisons.
\end{revision}
\begin{table}[H]
  \normalcolor
  \centering
  \small
  \begin{tabular}{rcccc}
    \toprule
    \textbf{Trajectories} & \multicolumn{3}{c}{$\lambda=0.1$} & $\lambda=0.01$ \\
    \cmidrule(lr){2-4}\cmidrule(lr){5-5}
    & \textbf{Jaccard} & \textbf{Spearman} & \textbf{Selected dims} & \textbf{Spearman} \\
    \midrule
    25 & 0.40 & 0.69 & 13.0 & 0.03 \\
    50 & 0.45 & 0.71 & 11.0 & 0.08 \\
    100 & 0.65 & 0.79 & 6.0 & 0.30 \\
    250 & 0.80 & 0.81 & 4.0 & 0.61 \\
    500 & 0.80 & 0.86 & 4.0 & 0.81 \\
    1000 & 0.87 & 0.94 & 4.3 & 0.93 \\
    2000 (all) & 1.00 & 0.99 & 5.0 & 0.98 \\
    \bottomrule
  \end{tabular}
  \caption{\rev{\textbf{Gate stability versus offline-data size.} Hinge with $\kappa=1$. Values are means over three trajectory-subset draws. The full-data gates from 2,000 trajectories are the reference for Jaccard and Spearman calculations.}}
  \label{tab:data_efficiency}
\end{table}
\begin{revision}
With $\lambda=0.1$, the selected-dimension set reaches 0.80 Jaccard overlap with the full-data reference using 250--500 trajectories, and the full-gate ranking reaches Spearman 0.94 at 1,000 trajectories. With weaker sparsity ($\lambda=0.01$), ranking is less stable at small sample sizes (Spearman 0.81 at 500 trajectories). Small subsets also retain more dimensions: 25 trajectories retain 13 on average, compared with 5 using all 2,000.
\end{revision}

\begingroup
\color{black}
\allowdisplaybreaks[1]
\section{State-Weight Comparison and Planning Proofs}
\label{app:swim-weighted-planning-proof}
We use the normalized learned weights
$\vg=\vg^*/\lVert\vg^*\rVert_1$ and reference weights
$\vg^0\ge0$, with $\sum_k g_k^0=1$, on the same Stage-2 standardized
state coordinates.
The notation $\vg$ denotes normalized weights in this appendix and
the planning analysis in Section~\ref{sec:theory}; elsewhere, it retains
its Stage-1 meaning as the sigmoid gates.
Normalization preserves candidate ordering and, for nonconstant
distance vectors, leaves ideal Pearson correlation unchanged.
The subscript of $D^2$ specifies the weights in the sum in
Equation~\eqref{eq:swim_distance}, replacing $\vg^*$ as appropriate.
Fix a start--goal pair and $M\ge2$ candidate action sequences, and write
$\vq_H^{(m)}\in\R^d$, with $d\ge1$, for candidate $m$'s actual terminal
state and $\vq_g$ for the goal state.

The following proposition isolates the effect of the state metric under
exact state-based candidate selection. These exact state-metric minimizers
are analytical benchmarks; the deployed planner searches using the
predicted latent cost $\widehat J_m$.
The comparison relates deviations from the segment joining uniform and
reference weights to the reduction in learned-metric cost.

\begin{proposition}[Comparison with uniform state weighting]
\label{prop:swim-weight-advantage}
Let $\vg^{\mathrm u}=(1/d,\ldots,1/d)$ denote uniform weights and set
$R=\max_{m,k}|q_{H,k}^{(m)}-q_{g,k}|$.
Choose any
$m_g\in\arg\min_m D_{\vg}^2(\vq_H^{(m)},\vq_g)$ and
$m_{\mathrm u}\in\arg\min_m D_{\vg^{\mathrm u}}^2(\vq_H^{(m)},\vq_g)$.
For $0<\beta\le1$, let $\varepsilon_\beta
=\lVert\vg-(1-\beta)\vg^{\mathrm u}-\beta\vg^0\rVert_1,\
\Delta_g
=D_{\vg}^2(\vq_H^{(m_{\mathrm u})},\vq_g)
-D_{\vg}^2(\vq_H^{(m_g)},\vq_g)
\ge0.$
The condition $R^2\varepsilon_\beta\le\Delta_g$ for at least one
$\beta\in(0,1]$ implies
\begin{equation}
D_{\vg^0}^2(\vq_H^{(m_g)},\vq_g)
\le D_{\vg^0}^2(\vq_H^{(m_{\mathrm u})},\vq_g),
\label{eq:swim-weight-advantage}
\end{equation}
with strict improvement whenever $R^2\varepsilon_\beta<\Delta_g$.
Moreover, every $\beta\in(0,1]$ satisfies
\begin{equation}
\lVert\vg-\vg^0\rVert_1
\le(1-\beta)\lVert\vg^{\mathrm u}-\vg^0\rVert_1+\varepsilon_\beta.
\label{eq:swim-weight-contraction}
\end{equation}
\end{proposition}

Uniform weights give the same candidate ordering as the unweighted state
distance used by SCALE. A reference metric supported on task-relevant
coordinates assigns zero weight to nuisance coordinates.
The parameter $\beta$ indexes a linear interpolation of the uniform and
reference weights, and $\varepsilon_\beta$ measures deviation from this
segment. The learned gates need not lie on the segment, but the sufficient
condition still restricts their deviation relative to the selection gain.
Zero deviation recovers the exact interpolation case.
At $\beta=1$, the sufficient condition becomes
$R^2\lVert\vg-\vg^0\rVert_1\le\Delta_g$.
The comparison uses exact state-metric minimization; Stage~1 is not
assumed to satisfy its sufficient condition, consistent with the gate
interpretation in Appendix~\ref{app:metric}.

\begin{remark}[Comparison with SCALE-All]
\label{rem:swim-scale-bound}
Theorem~\ref{thm:swim-weighted-planning} provides a common basis for
comparing \method{} and \scaleall{} when its conditions hold for both
models on the same candidates, standardized coordinates, and reference metric.
Let $\vg^{\mathrm u}=(1/d,\ldots,1/d)$ and
$\varepsilon_\beta=\lVert\vg-(1-\beta)\vg^{\mathrm u}-\beta\vg^0\rVert_1$
for $0<\beta\le1$.
By Proposition~\ref{prop:swim-weight-advantage}, the reduction in the
weight-discrepancy term from \scaleall{} to \method{} satisfies
\begin{equation}
R^2\bigl(\lVert\vg^{\mathrm u}-\vg^0\rVert_1
-\lVert\vg-\vg^0\rVert_1\bigr)
\ge R^2\bigl(\beta\lVert\vg^{\mathrm u}-\vg^0\rVert_1-\varepsilon_\beta\bigr).
\label{eq:swim-bound-improvement}
\end{equation}
The \method{} total upper bound is strictly smaller whenever the
right-hand side is positive and exceeds any increase in the combined
calibration and rollout terms.
\end{remark}

We first establish a weight-perturbation bound and prove the proposition.
We then control latent rollout error and prove
Theorem~\ref{thm:swim-weighted-planning}.

The first lemma controls the change in a candidate cost difference when
the learned weights are replaced by the reference weights.

\begin{lemma}[Perturbation of a candidate cost difference]
\label{lem:swim-weight-perturbation}
For candidates $i,j\in\{1,\ldots,M\}$, suppose
$|q_{H,k}^{(i)}-q_{g,k}|\le R$ and
$|q_{H,k}^{(j)}-q_{g,k}|\le R$ for all $k$. Then
\begin{align*}
\Big|&\{D_{\vg^0}^2(\vq_H^{(i)},\vq_g)-D_{\vg^0}^2(\vq_H^{(j)},\vq_g)\}\\
&-\{D_{\vg}^2(\vq_H^{(i)},\vq_g)-D_{\vg}^2(\vq_H^{(j)},\vq_g)\}\Big|
\le R^2\lVert \vg-\vg^0\rVert_1.
\end{align*}
\end{lemma}
\begin{proof}
Expand both cost differences and collect the terms involving the weight
change. The expression inside the absolute value becomes
\[
\sum_{k=1}^d(g_k^0-g_k)
\bigl\{(q_{H,k}^{(i)}-q_{g,k})^2
-(q_{H,k}^{(j)}-q_{g,k})^2\bigr\}.
\]
The coordinate bounds place both squared differences in $[0,R^2]$.
Their difference therefore has absolute value at most $R^2$.
Taking the absolute value of the sum and applying the triangle inequality
bounds it by
\[
R^2\sum_{k=1}^d|g_k^0-g_k|
=R^2\lVert\vg-\vg^0\rVert_1,
\]
which proves the stated inequality.
\end{proof}

\begin{proof}[Proof of Proposition~\ref{prop:swim-weight-advantage}]
Fix any minimizers $m_g$ and $m_{\mathrm u}$ and any $\beta\in(0,1]$.
The minimizing property of $m_g$ gives $\Delta_g\ge0$.
For this proof, write $\vg_\beta=(1-\beta)\vg^{\mathrm u}+\beta\vg^0$ for the
comparison weights, so that $\varepsilon_\beta=\lVert\vg-\vg_\beta\rVert_1$.
Applying Lemma~\ref{lem:swim-weight-perturbation} to the two selected
candidates, with $\vg_\beta$ in place of the reference weights, gives
\begin{align*}
&D_{\vg_\beta}^2(\vq_H^{(m_g)},\vq_g)
-D_{\vg_\beta}^2(\vq_H^{(m_{\mathrm u})},\vq_g)\\
&\quad\le D_{\vg}^2(\vq_H^{(m_g)},\vq_g)
-D_{\vg}^2(\vq_H^{(m_{\mathrm u})},\vq_g)
+R^2\varepsilon_\beta
=R^2\varepsilon_\beta-\Delta_g.
\end{align*}
The lemma's coordinate bounds follow from the definition of $R$.
Linearity in the weights expresses the left-hand side as
\begin{align*}
&(1-\beta)\bigl\{D_{\vg^{\mathrm u}}^2(\vq_H^{(m_g)},\vq_g)
-D_{\vg^{\mathrm u}}^2(\vq_H^{(m_{\mathrm u})},\vq_g)\bigr\}\\
&\quad+\beta\bigl\{D_{\vg^0}^2(\vq_H^{(m_g)},\vq_g)
-D_{\vg^0}^2(\vq_H^{(m_{\mathrm u})},\vq_g)\bigr\}.
\end{align*}
The first term is nonnegative because $m_{\mathrm u}$ minimizes the
uniform-metric cost and $1-\beta\ge0$. Dropping this term and dividing by
$\beta>0$ therefore yields
\begin{align*}
&D_{\vg^0}^2(\vq_H^{(m_g)},\vq_g)
-D_{\vg^0}^2(\vq_H^{(m_{\mathrm u})},\vq_g)\\
&\qquad\le\frac{R^2\varepsilon_\beta-\Delta_g}{\beta}.
\end{align*}
The condition $R^2\varepsilon_\beta\le\Delta_g$ makes the right-hand side
nonpositive, proving Equation~\eqref{eq:swim-weight-advantage}.
Strict inequality makes it negative and gives strict improvement.
The argument holds for every choice among tied minimizers.

Finally, decompose the weight error through the comparison weights:
\[
\vg-\vg^0=(\vg-\vg_\beta)+(1-\beta)(\vg^{\mathrm u}-\vg^0).
\]
The triangle inequality and $1-\beta\ge0$ give
\[
\lVert\vg-\vg^0\rVert_1
\le\lVert\vg-\vg_\beta\rVert_1
+(1-\beta)\lVert\vg^{\mathrm u}-\vg^0\rVert_1,
\]
which is Equation~\eqref{eq:swim-weight-contraction}.
\end{proof}

The next lemma converts an error in the terminal embedding into an error
in the squared planning cost.

\begin{lemma}[Rollout error in the planning cost]
\label{lem:swim-rollout-cost}
For every candidate $m$,
\[
|\widehat J_m-J_m^{\mathrm{true}}|
\le 2C_zC_{\mathrm{roll}}+C_{\mathrm{roll}}^2.
\]
\end{lemma}
\begin{proof}
Fix a candidate $m$. Factoring the difference of squared norms reduces the
claim to bounding the difference and the sum of the two norms.
The reverse triangle inequality and the definition of $C_{\mathrm{roll}}$
give
\[
\big|\lVert\widehat\vz_H^{(m)}-\vz_g\rVert_2
-\lVert\vz_H^{(m)}-\vz_g\rVert_2\big|
\le\lVert\widehat\vz_H^{(m)}-\vz_H^{(m)}\rVert_2
\le C_{\mathrm{roll}}.
\]
For their sum, the triangle inequality gives
\begin{align*}
&\lVert\widehat\vz_H^{(m)}-\vz_g\rVert_2
+\lVert\vz_H^{(m)}-\vz_g\rVert_2\\
&\qquad\le\lVert\widehat\vz_H^{(m)}-\vz_H^{(m)}\rVert_2
+2\lVert\vz_H^{(m)}-\vz_g\rVert_2
\le C_{\mathrm{roll}}+2C_z.
\end{align*}
Multiplying these bounds in the difference-of-squares factorization yields
\[
|\widehat J_m-J_m^{\mathrm{true}}|
\le C_{\mathrm{roll}}(2C_z+C_{\mathrm{roll}})
=2C_zC_{\mathrm{roll}}+C_{\mathrm{roll}}^2.
\]
This holds for every candidate $m$.
\end{proof}

\begin{proof}[Proof of Theorem~\ref{thm:swim-weighted-planning}]
Fix any $\widehat m\in\arg\min_m\widehat J_m$ and choose
$m^0\in\arg\min_mD_{\vg^0}^2(\vq_H^{(m)},\vq_g)$.
Both minimizers exist because the candidate set is finite and nonempty.
We first bound the difference between their learned-metric costs, then
use Lemma~\ref{lem:swim-weight-perturbation} to return to the reference metric.

\par\smallskip
\noindent\textbf{From correlation to cost differences.}
Let $\bar J$ and $\bar Q$ be the empirical means of
$(J_m^{\mathrm{true}})_{m=1}^M$ and $(Q_m)_{m=1}^M$, respectively.
Since $\sigma_z,\sigma_q>0$, we can define
\[
u_m=\frac{J_m^{\mathrm{true}}-\bar J}{\sigma_z},\qquad
v_m=\frac{Q_m-\bar Q}{\sigma_q},\qquad
e_m=v_m-u_m.
\]
The empirical standard deviations use divisor $M$, so their definitions
and the definition of Pearson correlation imply
\[
\sum_{m=1}^M u_m^2=\sum_{m=1}^M v_m^2=M,
\qquad
\sum_{m=1}^M u_mv_m=M\rho_{\mathrm P}.
\]
Expanding the squared residual and substituting these identities gives
\begin{align*}
\sum_{m=1}^M e_m^2
&=\sum_{m=1}^M v_m^2+\sum_{m=1}^M u_m^2
-2\sum_{m=1}^M u_mv_m\\
&=2M(1-\rho_{\mathrm P}).
\end{align*}
In particular, $M(1-\rho_{\mathrm P})\ge0$. For distinct indices $i,j$,
$(e_i-e_j)^2\le2(e_i^2+e_j^2)$, and the remaining squared residuals
are nonnegative. Hence
\[
(e_i-e_j)^2
\le2\sum_{m=1}^M e_m^2
=4M(1-\rho_{\mathrm P}).
\]
Taking square roots gives
$|e_i-e_j|\le2\sqrt{M(1-\rho_{\mathrm P})}$.
When $i=j$, the left-hand side is zero, so this bound holds for every pair.
The centering terms cancel in cost differences; substituting the definitions
of $u_m,v_m,e_m$ therefore gives the exact identity
\[
Q_{\widehat m}-Q_{m^0}
=\frac{\sigma_q}{\sigma_z}
 (J_{\widehat m}^{\mathrm{true}}-J_{m^0}^{\mathrm{true}})
+\sigma_q(e_{\widehat m}-e_{m^0}).
\]
This identity separates the true latent-cost difference from the
calibration residual, which the preceding bound controls.

\par\smallskip
\noindent\textbf{Cost of the selected candidate.}
To bound the true latent-cost difference, insert the predicted costs used
by the planner:
\begin{align*}
J_{\widehat m}^{\mathrm{true}}-J_{m^0}^{\mathrm{true}}
&=(J_{\widehat m}^{\mathrm{true}}-\widehat J_{\widehat m})
+(\widehat J_{\widehat m}-\widehat J_{m^0})\\
&\quad+(\widehat J_{m^0}-J_{m^0}^{\mathrm{true}})\\
&\le|J_{\widehat m}^{\mathrm{true}}-\widehat J_{\widehat m}|
+|\widehat J_{m^0}-J_{m^0}^{\mathrm{true}}|\\
&\le2(2C_zC_{\mathrm{roll}}+C_{\mathrm{roll}}^2).
\end{align*}
The first inequality uses $\widehat J_{\widehat m}\le\widehat J_{m^0}$;
the second applies Lemma~\ref{lem:swim-rollout-cost} to the two candidates.
Substituting this bound and the residual bound into the preceding cost
identity, and using $\sigma_q>0$ and $\sigma_q/\sigma_z>0$, yields
\begin{align*}
Q_{\widehat m}-Q_{m^0}
&\le2\sigma_q\sqrt{M(1-\rho_{\mathrm P})}\\
&\quad+2\frac{\sigma_q}{\sigma_z}
 (2C_zC_{\mathrm{roll}}+C_{\mathrm{roll}}^2).
\end{align*}

\par\smallskip
\noindent\textbf{Returning to the reference metric.}
The definition of $R$ supplies the coordinate bounds required by
Lemma~\ref{lem:swim-weight-perturbation}. Applying that lemma with
$i=\widehat m$ and $j=m^0$, and using the definition of $Q_m$, gives
\begin{align*}
&D_{\vg^0}^2(\vq_H^{(\widehat m)},\vq_g)
-D_{\vg^0}^2(\vq_H^{(m^0)},\vq_g)\\
&\qquad\le Q_{\widehat m}-Q_{m^0}
+R^2\lVert\vg-\vg^0\rVert_1.
\end{align*}
Substituting the learned-metric bound proves the upper inequality in
Equation~\eqref{eq:swim-weighted-planning}. The minimizing property of
$m^0$ gives the lower inequality. Since $\widehat m$ was any minimizer of
the predicted costs, the conclusion also covers arbitrary ties.
\end{proof}

\par
\endgroup

\end{document}